\documentclass[11pt]{article}
\usepackage{indentfirst}
\usepackage{fullpage}
\usepackage[numbers]{natbib}

\usepackage{hyperref}
\usepackage{url}
\usepackage[utf8]{inputenc} 
\usepackage[T1]{fontenc}    
\usepackage{hyperref}       
\usepackage{url}            
\usepackage{booktabs}       
\usepackage{amsfonts}       
\usepackage{nicefrac}       
\usepackage{microtype}      

\usepackage{amsmath}
\usepackage{amsfonts}
\usepackage{lastpage}
\usepackage{graphicx}
\usepackage{subfigure}
\usepackage{multirow}
\usepackage{tabu}
\usepackage{lipsum}

\usepackage{multirow}
\usepackage{tabu}
\usepackage{lipsum}
\usepackage{essay-def}
\usepackage{essay-def-aug}
\usepackage{caption}
\usepackage{appendix}
\usepackage{mathrsfs}
\usepackage{indentfirst}
\usepackage{natbib}

\usepackage{xcolor}
\definecolor{SchoolColor}{rgb}{0.6471, 0.1098, 0.1882} 
\definecolor{forestgreen}{rgb}{0.0, 0.27, 0.13}
\usepackage{tikz}
\title{Polynomial-Time Solutions for ReLU Network Training: A Complexity Classification via Max-Cut and Zonotopes}
\author{
Yifei Wang \\
  Department of Electrical Engineering\\
  Stanford University\\
  Stanford, CA 94305, USA\\
  \texttt{wangyf18@stanford.edu} \\
  \and
  Mert Pilanci \\
  Department of Electrical Engineering\\
  Stanford University\\
  Stanford, CA 94305, USA\\
  \texttt{pilanci@stanford.edu} \\
}

\begin{document}
\maketitle

\begin{abstract}
We investigate the complexity of training a two-layer ReLU neural network with weight decay regularization. Previous research has shown that the optimal solution of this problem can be found by solving a standard cone-constrained convex program. Using this convex formulation, we prove that the hardness of approximation of ReLU networks not only mirrors the complexity of the Max-Cut problem but also, in certain special cases, exactly corresponds to it. In particular, when $\epsilon\leq\sqrt{84/83}-1\approx 0.006$, we show that it is NP-hard to find an approximate global optimizer of the ReLU network objective with relative error $\epsilon$ with respect to the objective value. Moreover, we develop a randomized algorithm which mirrors the Goemans-Williamson rounding of semidefinite Max-Cut relaxations. To provide polynomial-time approximations, we classify training datasets into three categories: (i) For orthogonal separable datasets, a precise solution can be obtained in polynomial-time. (ii) When there is a negative correlation between samples of different classes, we give a polynomial-time approximation with relative error $\sqrt{\pi/2}-1\approx 0.253$. (iii) For general datasets, the degree to which the problem can be approximated in polynomial-time is governed by a geometric factor that controls the diameter of two zonotopes intrinsic to the dataset. To our knowledge, these results present the first polynomial-time approximation guarantees along with first hardness of approximation results for regularized ReLU networks. 
\end{abstract}



\noindent \textbf{Keywords:} Neural networks, convex optimization, Max-Cut, NP-hardness

\section{Introduction}

The meteoric rise of deep learning has confirmed the optimization and generalization abilities of neural networks trained using simple gradient-based heuristics. This success, however, presents a stark contrast to our theoretical understanding of these non-convex optimization problems, which suggests that neural network training is NP-hard. This discrepancy raises important questions about why these hard problems are routinely solved in practice with relatively straightforward optimization techniques. In this work, we aim to provide a nuanced characterization of the hardness of neural network training by establishing an unexplored connection to the Max-Cut problem.

In recent work \cite{nnacr,tcrnn}, the training of 2-layer ReLU neural networks with weight-decay regularization was reformulated as a cone-constrained convex program. Moreover, the set of globally optimal solutions can be characterized by the optimal set of the convex program \cite{wang2020hidden}. However, these approaches have focused on the global optimization of the objective  and have not considered the complexity of approximating the optimal value of the non-convex training problem. 

In this paper, we focus on determining the weights of a two-layer ReLU neural network which minimize an empirical loss. We show that approximating the optimal value of this optimization program is NP-hard for a worst-case instance of the dataset. The difficulty of the problem is mainly due to the dual constraints in the optimization program, which involves a subproblem with the same formulation of the well-known NP-hard Max-Cut problem. However, we present polynomial-time approximation schemes under certain assumptions on the training dataset. Our contributions can be summarized as follows:

\begin{itemize}
\item For structured datasets including orthogonal separable datasets, we can find the exact solution in polynomial-time by solving a convex program.
\item For datasets which have negative correlation between samples from different classes, we can find an $1-\sqrt{2/\pi}$ approximation of the problem in polynomial-time.
\item For general datasets, there is a polynomial-time algorithm whose approximation factor depends on a geometric ratio which is determined by the dataset. 
\end{itemize}

We summarize our results in Figure \ref{fig:tikz}.

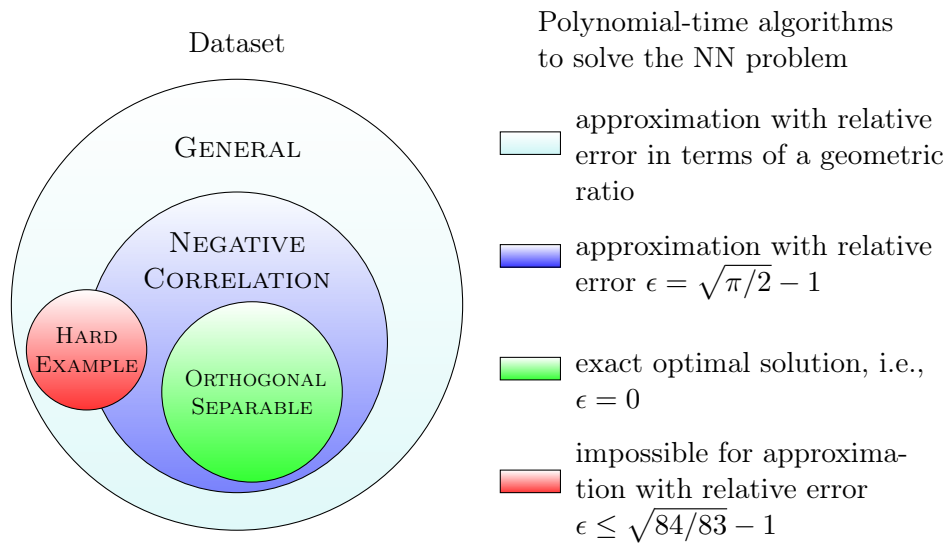
\begin{figure}[htbp]
\centering
\def\firstcircle{(0,0) circle (3.0cm)}
\def\secondcircle{(360:3.9cm) circle (3.0cm)}
\begin{tikzpicture}
\begin{scope}[fill opacity=0.60,text opacity=1,white,align=center,text width=2.5cm]
\draw[style=shade, top color=white, bottom color=cyan!80!black!20 ][black] \firstcircle 
node[shift={(0,2.1)},text=black]{\textsc{General}};
\begin{scope}[fill opacity=0.60,text opacity=1,white,align=center,text width=2.5cm]
\draw[style=shade, top color=white, bottom color=blue!80!][black] {(0,-0.5) circle (2.0cm)}
node[shift={(0,1.1)},text=black]{\textsc{Negative Correlation}};
\begin{scope}[scale=0.8,transform shape,align=center,white,fill opacity=1]
\draw[style=shade, top color=white, bottom color=green!80][black](0.25,-1.45)circle(1.50cm) 
 node[text width=2.25cm,text=black] {\textsc{Orthogonal\\ Separable}};
 \draw[style=shade, top color=white, bottom color=red!80][black](-2.5,-0.75)circle(1cm) 
 node[text width=2.25cm,text=black] {\textsc{Hard\\ Example}};
\end{scope}
\end{scope}
\end{scope}
\node[] at (0,3.5) {Dataset};
\draw[style=shade, top color=white, bottom color=cyan!80!black!20 ][black] (3.5,2)  {rectangle ++ (0.8,0.3)};
\draw[style=shade, top color=white, bottom color=blue!80!][black] (3.5,0.5)  {rectangle ++ (0.8,0.3)};
\draw[style=shade, top color=white, bottom color=green!80!][black] (3.5,-1)  {rectangle ++ (0.8,0.3)};
\draw[style=shade, top color=white, bottom color=red!80!][black] (3.5,-2.5)  {rectangle ++ (0.8,0.3)};
\begin{scope}[fill opacity=0.60,text opacity=1,white,align=left,text width=5cm]
\node[] at (7,2) [text=black]{approximation with relative error in terms of a geometric ratio};
\node[] at (7,0.5) [text=black]{approximation with relative error $\epsilon=\sqrt{\pi/2}-1$};
\node[] at (7,-1) [text=black]{exact optimal solution, i.e., $\epsilon=0$};
\node[] at (7,-2.5) [text=black]{impossible for approximation with relative error $\epsilon\leq \sqrt{84/83}-1$};
\node[] at (6.5,3.5) [text=black]{Polynomial-time algorithms to solve the NN problem};
\end{scope}
\end{tikzpicture}
\caption{Difficulty of approximation of the neural network max-margin problem using a polynomial-time algorithm. }\label{fig:tikz}
\end{figure}

\subsection{Related work}

Previous works including \cite{froese2023training,arora2016understanding} mainly studied NP hardness on the empirical minimization problem of two-layer ReLU neural networks without regularization. \cite{dey2020approximation,goel2020tight} studies the hardness of training 2-layer ReLU neural networks with single neuron. \cite{boob2022complexity} studies neural networks with output ReLU neurons and prove the NP hardness with two-neurons. \cite{bertschinger2022training} shows that training 2-layer ReLU networks with two output and two input neurons is NP hard. For the squared loss, assuming the existence of local pseudorandom generators, \cite{daniely2021local} proves the hardness of learning 2-layer ReLU networks with $w(1)$ neurons and the lower bound of $\Omega(n^{\beta k})$ to learn a 2-layer ReLU network with $k$ neurons where $n$ is the number of data and $\beta>0$ is an absolute constant. In the same setting, \cite{daniely2023computational} shows the hardness of learning 2-layer ReLU networks where random noises are added to the network's parameter and the input distribution. In the same setting, \cite{chen2022hardness} provides a statistical query lower bounds for learning 2-layer ReLU networks with respect to Gaussian inputs assuming the hardness of Learning with Rounding. Under specific assumptions on the input distribution and network's weight, polynomial-time algorithms to approximate 2-layer ReLU networks can be obtained \cite{awasthi2021efficient,bakshi2019learning,ge2018learning,ge2019learning,janzamin2015beating}. In particular, \cite{brutzkus2018sgd} show that linear separability of the dataset is sufficient. In comparison, we study the training problem of 2-layer ReLU neural networks with weight decay regularization and consider a generic convex loss including hinge loss and the max-margin problem. Because of the existence of regularization, we can analyze the problem by its convex dual problem, which is equivalent to the primal problem with a sufficiently large number of neurons.

\subsection{Overview of our results}

Let $X\in\mbR^{n\times d}$ and $y\in\{-1,1\}^n$ be the training data and training label. We focus on the following two-layer neural network architectures:

\begin{itemize}
\item ReLU activation: 
$$
f^\text{ReLU}(X;\Theta)=\sum_{i=1}^m(Xw_{1,i})_+w_{2,i},
$$
where $\Theta=(W_1,w_2)$, $W_1\in\mbR^{d\times m}$ and $w_2\in\mbR^m$. The weight decay regularization is given by $R(\Theta)=\frac{1}{2}(\|W_1\|_F^2+\|w_2\|_2^2)$.
\item Gated ReLU (gReLU) activation: 
$$
f^\text{gReLU}(X;\Theta)=\sum_{i=1}^m\mbI(Xh_i\geq0)Xw_{1,i}w_{2,i},
$$
where $\Theta=(H,W_1,w_2)$, $H,W_1\in\mbR^{d\times m}$ and $w_2\in\mbR^m$. Here $\mbI(s)=1$ if the statement $s$ holds and $\mbI(s)=0$ otherwise. The weight decay regularization is given by $R(\Theta)=\frac{1}{2}(\|W_1\|_F^2+\|w_2\|_2^2)$. 
\end{itemize}

We first consider the following problem for binary classification, referred to as max-margin problem
\begin{equation}\label{maxmargin:primal}
\begin{aligned}
   P_\text{non-cvx}= \min_{\Theta}&\; R(\Theta)\\
   \text{ s.t. }& y_if^\text{ReLU}(x_i;\Theta)\geq 1, \forall i\in[n].
\end{aligned}
\end{equation}
Here $f^\text{ReLU}(X;\Theta)$ is the output of a ReLU network. We formulate the neural network training problem as the following decision problem.

\texttt{Approx-Primal}
\begin{itemize}
\item Input: data pair $X\in\mbR^{n\times d}$ and $y\in\{-1,1\}^n$, relative error $\epsilon\geq 0$.
\item Goal: find a value $p$ such that $P_\text{non-cvx}\leq p\leq (1+\epsilon)P_\text{non-cvx}$ and find a neural network with parameter $\Theta$ such that $p=R(\Theta)$ and $y_if(X;\Theta)\geq 1$ for all $i\in[n]$. 
\end{itemize}

We next consider the hinge loss minimization:
\begin{equation}\label{maxmargin:primal_hinge}
\begin{aligned}
   P_\text{non-cvx}^\text{hinge}= \min_{\Theta}&\; \sum_{i=1}^n\max\{0,1-y_if^\text{ReLU}(x_i;\Theta)\} +\beta R(\Theta),
\end{aligned}
\end{equation}
where $\beta>0$ is a regularization parameter. The associated decision problem is presented as follows.

\texttt{Approx-Primal-Hinge}
\begin{itemize}
\item Input: data pair $X\in\mbR^{n\times d}$ and $y\in\{-1,1\}^n$, $\beta>0$, relative error $\epsilon\geq 0$.
\item Goal: find a value $p$ such that $P_\text{non-cvx}^\text{hinge}\leq p\leq (1+\epsilon)P_\text{non-cvx}^\text{hinge}$ and find a neural network with parameter $\Theta$ such that $p=\sum_{i=1}^n\max\{0,1-y_if^\text{ReLU}(x_i;\Theta)\} +\beta R(\Theta)$. 
\end{itemize}

We can also extend the negative result to training problems with general loss function: \begin{equation}\label{maxmargin:primal_general}
\begin{aligned}
   P_\text{non-cvx}^\text{gen}= \min_{\Theta}&\; \sum_{i=1}^n\ell\pp{y_if^\text{ReLU}(x_i;\Theta)} +\beta R(\Theta),
\end{aligned}
\end{equation}
where $\beta>0$ is a regularization parameter. Here $\ell(z)$ takes the form $\ell(z)=\max_{\lambda:\lambda\geq 0}-\lambda^Tz+g(\lambda)$. Here, $g(\lambda):=-\ell^*(-\lambda)$ where $\ell^*$ is the convex conjugate of $\ell$.

\texttt{Approx-Primal-General}
\begin{itemize}
\item Input: data pair $X\in\mbR^{n\times d}$ and $y\in\{-1,1\}^n$, $\beta>0$, relative error $\epsilon\geq 0$.
\item Goal: find a value $p$ such that $P_\text{non-cvx}^\text{gen}\leq p\leq (1+\epsilon)P_\text{non-cvx}^\text{gen}$ and find a neural network with parameter $\Theta$ such that $p=\sum_{i=1}^n l\pp{y_if^\text{ReLU}(x_i;\Theta)} +\beta R(\Theta)$. 
\end{itemize}

We summarize our results for the difficulties of solving \texttt{Approx-Primal} and \texttt{Approx-Primal-Hinge} as follows. 

\begin{theorem}[main results]\label{thm:informal_primal}
\begin{itemize}
\item Suppose that $P\neq NP$ and set a relative error $\epsilon\leq \sqrt{84/83}-1$. Then, there does not exist a polynomial-time algorithm to solve \texttt{Approx-Primal} or \texttt{Approx-Primal-Hinge}. For a generic loss $\ell$ for which the conjugate $\ell^*(\lambda)=-g(-\lambda)$ satisfies $g(a\lambda)\geq a g(\lambda),\forall a>1$, there does not exist a polynomial-time algorithm to solve \texttt{Approx-Primal-General}.
\item Suppose that $(X,y)$ is orthogonal separable, i.e., 
\begin{equation*}
x_i^Tx_j\geq 0, \text{ if }y_i=y_j,\quad x_i^Tx_j\leq 0,  \text{ if } y_i\neq y_j.
\end{equation*}
Then, for arbitrary $\epsilon>0$, \texttt{Approx-Primal}, \texttt{Approx-Primal-Hinge} and \texttt{Approx-Primal-General} can be solved in polynomial-time.
\item Suppose that $(X,y)$ has negative correlation, i.e., $x_ix_j\leq 0$ for $y_i\neq y_j$. Let $\delta>0$. For $\epsilon=\sqrt{\pi/2}-1$, there exists a polynomial-time algorithm to solve the \texttt{Approx-Primal} problem or the \texttt{Approx-Primal-Hinge} problem with probability at least $1-\delta$. For generic loss functions whose conjugate $-\ell^*(-\lambda)=g(\lambda)$ satisfies $g(a\lambda)\geq C a g(\lambda),\forall a\in(0,1]$ with parameter $C\in(0,1]$, with $\epsilon=C^{-1}\sqrt{\pi/2}-1$, there exists a polynomial-time algorithm to solve the \texttt{Approx-Primal-General} problem with probability at least $1-\delta$.
\item Let $c\in(0,1)$. Suppose that $c\leq \min\{c^*,(c^*)^{-1}\}$, where $c^*$ is a geometric ratio defined in Definition \ref{def:ratio}. We can find a value $p$ to solve \texttt{Approx-Primal} and \texttt{Approx-Primal-Hinge} in polynomial-time with relative error $\epsilon=(1-c)^{-1}\sqrt{\pi/2} -1$. For generic loss functions whose conjugate $-\ell^*(-\lambda)=g(\lambda)$ satisfies  $g(a\lambda)\geq aC g(\lambda),\forall a\in(0,1]$ with parameter $C\in(0,1]$, we can find a value $p$ to solve \texttt{Approx-Primal-General} in polynomial-time with relative error $\epsilon=C^{-1}(1-c)^{-1}\sqrt{\pi/2} -1$.
\end{itemize}
\end{theorem}

The above result is presented in further detail and proved via Theorems \ref{thm:primal_neg}, \ref{thm:ortho_primal}, \ref{thm:neg_cor_primal} and \ref{thm:geo_bnd}.

\section{Convex duality of ReLU networks and connections to zonotopes}\label{sec:dual}

The convex duality plays an important role in analyzing the optimal layer weight of two-layer neural networks with ReLU activations \cite{nnacr, ergen2020aistats, ergen2020jmlr,ergen2021revealing}. To investigate on the complexity of approximating the optimal value of \eqref{maxmargin:primal}, we first study the convex dual problem of \eqref{maxmargin:primal}. According to \cite{wang2021parallel}, the dual problem of \eqref{maxmargin:primal} is given by
\begin{equation}\label{maxmargin:dual}
    D=\max \lambda^Ty \text{ s.t. } \diag(y)\lambda \geq 0, \max_{u:\|u\|_2\leq 1} |\lambda^T(Xu)_+|\leq 1.
\end{equation}
The dual problem of \eqref{maxmargin:primal_hinge} and \eqref{maxmargin:primal_general} is given by
\begin{equation}\label{maxmargin:dual_general}
    D^\text{gen}=\max g(\diag(y)\lambda) \text{ s.t. } \diag(y)\lambda \geq 0, \max_{u:\|u\|_2\leq 1} |\lambda^T(Xu)_+|\leq \beta.
\end{equation}
Here we write $g(\lambda)=\sum_{i=1}^ng(\lambda_i)$ for simplicity. From \cite{nnacr}, for $m\geq m^*$, where $m^*\leq n+1$ is a critical value, for ReLU networks, there is no duality gap. In other words, we have $P_\text{non-cvx}=D$, $P_\text{non-cvx}^\text{hinge}=D^\text{hinge}$ and $P_\text{non-cvx}^\text{gen}=D^\text{gen}$
Then, we also formulate the dual problems as decision problems. 

 \texttt{Approx-Dual}
\begin{itemize}
\item Input: data pair $X\in\mbR^{n\times d}$ and $y\in\{-1,1\}^n$, relative error $\epsilon\geq 0$.
\item Goal: find $\lambda\in\mbR^n$ such that $\diag(y)\lambda \geq 0, \max_{u:\|u\|_2\leq 1} |\lambda^T(Xu)_+|\leq 1$ and $\lambda^Ty\geq (1-\epsilon)D$?
\end{itemize}

 \texttt{Approx-Dual-Hinge}
\begin{itemize}
\item Input: data pair $X\in\mbR^{n\times d}$ and $y\in\{-1,1\}^n$, $\beta>0$, relative error $\epsilon\geq 0$.
\item Goal: find $\lambda\in\mbR^n$ such that $0\leq \diag(y)\lambda \leq 1, \max_{u:\|u\|_2\leq 1} |\lambda^T(Xu)_+|\leq \beta $ and $\lambda^Ty\geq (1-\epsilon)D^\text{hinge}$?
\end{itemize}

 \texttt{Approx-Dual-General}
\begin{itemize}
\item Input: data pair $X\in\mbR^{n\times d}$ and $y\in\{-1,1\}^n$, $\beta>0$, relative error $\epsilon\geq 0$.
\item Goal: find $\lambda\in\mbR^n$ such that $\diag(y)\lambda \geq 0, \max_{u:\|u\|_2\leq 1} |\lambda^T(Xu)_+|\leq \beta $ and $g(\diag(y)\lambda)\geq (1-\epsilon)D^\text{gen}$?
\end{itemize}

We summarize our main theoretical results for the difficulties of solving \texttt{Approx-Dual}, \texttt{Approx-Dual-Hinge} and \texttt{Approx-Dual-General} with different assumptions on the datasets as follows.

\begin{theorem}\label{thm:informal_dual}
\begin{itemize}
\item Suppose that $P\neq NP$. Set a relative error $\epsilon\leq 1-\sqrt{83/84}$. Then, there does not exist a polynomial-time algorithm to solve \texttt{Approx-Dual} or \texttt{Approx-Dual-Hinge}. For $g(\lambda)$ satisfies $g(a\lambda)\geq a g(\lambda),\forall a>1$, there does not exist a polynomial-time algorithm to solve \texttt{Approx-Dual-General}.
\item Suppose that $(X,y)$ is orthogonal separable. For arbitrary $\epsilon>0$, there exist a polynomial-time algorithm to solve \texttt{Approx-Dual}, \texttt{Approx-Dual-Hinge} and and \texttt{Approx-Dual-General}.
\item Suppose that $(X,y)$ has negative correlation. For $\epsilon=1-\sqrt{2/\pi}$, there exist a polynomial-time algorithm to solve \texttt{Approx-Dual} and \texttt{Approx-Dual-Hinge}. For $g(\lambda)$ satisfying $g(a\lambda)\geq aC g(\lambda),\forall a\in(0,1]$ with parameter $C\in(0,1]$, with $\epsilon=1-C^{-1}\sqrt{2/\pi}$, there exist a polynomial-time algorithm to solve \texttt{Approx-Dual-General}.
\item Let $c\in(0,1)$. Suppose that $c\leq \min\{c^*,(c^*)^{-1}\}$, where $c^*$ is a geometric ratio defined in Definition \ref{def:ratio}. Then, with relative error $\epsilon=1-(1-c)\sqrt{\pi/2}$, there exists a polynomial-time algorithm to solve \texttt{Approx-Dual} and \texttt{Approx-Dual-Hinge}.  For $g(\lambda)$ satisfying $g(a\lambda)\geq aC g(\lambda),\forall a\in(0,1]$ with parameter $C\in(0,1]$, with $\epsilon=1-(1-c)\sqrt{\pi/2}$, there exist a polynomial-time algorithm to solve \texttt{Approx-Dual-General}.
\end{itemize}
\end{theorem}
The full version of Theorem \ref{thm:informal_dual} is stated and proved in Theorems \ref{thm:dual_neg}, \ref{thm:ortho_dual}, \ref{thm:neg_cor_dual} and \ref{thm:geo_bnd}.

For simplicity, we write $X_+=\{x_i\}_{i:y_i=1}\in\mbR^{n_+\times d}$ and $X_-=\{x_i\}_{i:y_i=-1}\in\mbR^{n_-\times d}$, where $n_+=\#\{i:y_i=1\}$ and $n_-=\#\{i:y_i=-1\}$. We also write $\lambda_+=\{\lambda_i\}_{i:y_i=1}\in\mbR^{n_+}$ and $\lambda_-=\{-\lambda_i\}_{i:y_i=-1}\in\mbR^{n_-}$. The following lemma characterizes the dual constraint $\max_{u:\|u\|_2\leq 1} |\lambda^T(Xu)_+|\leq 1$ by two maximin problems. 

\begin{lemma}\label{lem:maximin}
For fixed $\lambda\in\mbR^d$, the dual constraint $\max_{u:\|u\|_2\leq 1} |\lambda^T(Xu)_+|\leq 1$ is equivalent to 
\begin{equation}
\begin{aligned}
    \max_{b_+\in[0,1]^{n_+}}\min_{b_-\in[0,1]^{n_-}}\|X_+^T\diag(\lambda_+)b_+-X_-^T\diag(\lambda_-)b_-\|_2\leq 1,\\
    \max_{b_-\in[0,1]^{n_-}}\min_{b_+\in[0,1]^{n_+}}\|X_+^T\diag(\lambda_+)b_+-X_-^T\diag(\lambda_-)b_-\|_2\leq 1.
\end{aligned}
\end{equation}
\end{lemma}

To under the maximin problems induced by the dual constraint, we introduce the definition of zonotopes. 

\begin{definition}
Let $A\in \mbR^{n\times m}$. Define the zonotope $\mcZ(A)\in \mbR^n$ as the convex set
\begin{equation}
    \mcZ(A)=\{Au|u\in[0,1]^m\},
\end{equation}
which is a linear embedding of the unit hypercube $[0,1]^n$. Figure \ref{fig:zonotope} depicts the zonotope generated by 8 random vectors.
\end{definition}

\begin{figure}
\centering
\begin{minipage}[t]{0.5\textwidth}
\centering
\includegraphics[width=\linewidth]{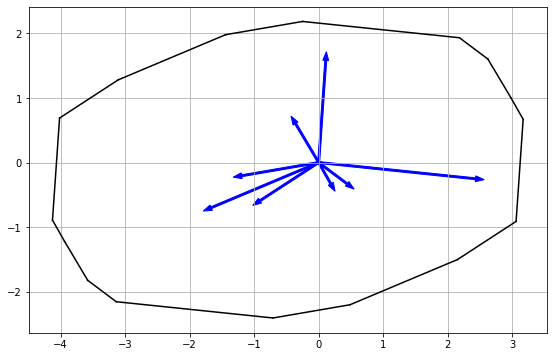}
\end{minipage}
\caption{An example of zonotope $\mcZ(A)$ where $A\in\mbR^{2\times 8}$. The blue arrows represent columns of $A$, which are randomly sampled from $\mcN(0,1)$. \label{fig:zonotope} }
\end{figure}

As a direct corollary of Lemma \ref{lem:maximin}, the following proposition characterizes the dual constraint 
in terms of the Haussdorf distance between two zonotopes.

\begin{proposition}\label{prop:hauss_dist}Consider the zonotopes 
\begin{equation}\label{equ:zono}
\begin{aligned}
    K_+&=Z(X_+^T\diag(\lambda_+))=\{X_+^T\diag(\lambda_+)b_+|b_+\in[0,1]^{n_+}\},\\
    K_-&=Z(X_-^T\diag(-\lambda_-))=\{X_-^T\diag(\lambda_-)b_-|b_-\in[0,1]^{n_-}\}.
\end{aligned}
\end{equation}
Then, the dual constraint $\max_{u:\|u\|_2\leq 1} |\lambda^T(Xu)_+|\leq 1$ is equivalent to
\begin{equation}
    H(K_+,K_-)\leq 1,
\end{equation}
where $H(K_+,K_-)$ is the Haussdorff distance between two convex sets $K_+,K_-\subseteq\mbR^n$ defined by
\begin{equation}
\begin{aligned}
H(K_+,K_-)=\max\big\{\max_{u_-\in K_-}\min_{u_+\in K_+}\|u_+-u_-\|_2,\; \max_{u_+\in K_+}\min_{u_-\in K_-}\|u_+-u_-\|_2\big\}.
\end{aligned}
\end{equation}
\end{proposition}

\section{Connections to maxcut}

The reformulation of the dual constraint in terms of Haussdorff distance between two zonotopes allows us to relate the dual problem with the well-known Max-Cut problem.  Consider a special case for the dual problem \eqref{maxmargin:dual} when $y=\bone$. In this case, we have $K_-=\{0\}$. Then, the constraint $\diag(y)\lambda\geq 0$ reduces to $\lambda\geq 0$. In this case, we have
\begin{equation}
\begin{aligned}
    &\max_{u:\|u\|_2\leq 1} |\lambda^T(Xu)_+|=H(K_+,0)=\max_{z\in[0,1]^n} \|X^T\diag(\lambda)z\|_2.
\end{aligned}
\end{equation}
Thus, the dual problem \eqref{maxmargin:dual} is also equivalent to
\begin{equation}
    \max \lambda^Ty \text{ s.t. } \diag(y)\lambda \geq 0, \max_{z\in[0,1]^n} \|X^T\diag(\lambda)z\|_2^2\leq 1.
\end{equation}
An interesting observation is that for fixed $\lambda$, we have
\begin{equation}\label{equ:connect_maxcut}
\begin{aligned}
   \max_{z\in[0,1]^n}z^T\diag(\lambda)XX^T\diag(\lambda)z
\stepa{=}&\max_{z\in\{0,1\}^n}z^T\diag(\lambda)XX^T\diag(\lambda)z\\
    =& \frac{1}{4}\max_{z\in\{-1,1\}^n}\bmbm{z\\1}^T\bmbm{I\\\bone^T}\diag( \lambda)XX^T\diag(  \lambda)\bmbm{I&\bone}\bmbm{z\\1}\\
    =&\frac{1}{4}\max_{\tilde z\in\{-1,1\}^{n+1}}\tilde z^T\bmbm{I\\\bone^T}\diag( \lambda)XX^T\diag( \lambda)\bmbm{I&\bone}\tilde z.
\end{aligned}
\end{equation}
Here, (a) follows since the maximizers of a convex function over a convex set is an extreme point. From \eqref{equ:connect_maxcut}, by taking $W=\bmbm{I\\\bone^T}\diag( \lambda)XX^T\diag( \lambda)\bmbm{I&\bone}$, we can rewrite the dual constraint via the Max-Cut value
$$
\max_{z\in\{-1,1\}^{n+1}} z^TWz\le 1.
$$
We note that the above Max-Cut problem is NP-hard in general. Therefore, we develop negative results to solve \texttt{Approx-Dual} and \texttt{Approx-Primal} with arbitrarily small relative error $\epsilon>0$ as follows.

\subsection{Negative result}\label{sec:neg_result}

\begin{theorem}\label{thm:dual_neg}
Suppose that $P\neq NP$. Let $y=\bone$ and set a relative error $\epsilon\leq 1-\sqrt{83/84}$. Then, there does not exist a polynomial-time algorithm to solve the \texttt{Approx-Dual} problem or the \texttt{Approx-Dual-Hinge} problem. For $g(\lambda)$ satisfies $g(a\lambda)\geq a g(\lambda),\forall a>1$, there does not exist a polynomial-time algorithm to solve \texttt{Approx-Dual-General}.
\end{theorem}

As there is no duality gap, we can also show that there does not exist a polynomial algorithm to solve \texttt{Approx-Primal}, \texttt{Approx-Primal-Hinge} or \texttt{Approx-Primal-General} with arbitrary small relative error.

\begin{theorem}\label{thm:primal_neg}
Suppose that $P\neq NP$.  Let $y=\bone$ and set a relative error $\epsilon\leq \sqrt{84/83}-1$. Then, there does not exist a polynomial-time algorithm to solve \texttt{Approx-Primal} or \texttt{Approx-Primal-Hinge}. For $g(\lambda)$ satisfies $g(a\lambda)\geq a g(\lambda),\forall a>1$, there does not exist a polynomial-time algorithm to solve \texttt{Approx-Dual-General}.
\end{theorem}

\section{Exact computation for orthogonal separable dataset}\label{sec:ortho}

We first consider a special class of input data pair $X,y$, which is orthogonal separable.  We say that the dataset $(X,y)$ is orthogonal separable if for all $i,j\in[n]$,
\begin{equation*}
\begin{aligned}
&x_i^Tx_j\geq 0, \text{ if }y_i=y_j,\quad x_i^Tx_j\leq 0,  \text{ if } y_i\neq y_j.
\end{aligned}
\end{equation*}
In short, the dataset $(X,y)$ is orthogonal separable implies that
$$
X_+X_+^T\geq 0, X_-X_+^T\leq 0, X_-X_-^T\geq 0.
$$
For the orthogonal separable dataset, even with relative error $\epsilon=0$, we can solve \texttt{Approx-Dual} in polynomial-time.

\begin{theorem}\label{thm:ortho_dual}
Suppose that $(X,y)$ is orthogonal separable. The dual problem \eqref{maxmargin:dual} is equivalent to
\begin{equation}\label{dual:ortho}
    \max \lambda_+^T\bone+\lambda_-\bone \text{ s.t. } \lambda_+\geq0, \lambda_-\geq0, \|X_+^T\lambda_+\|_2\leq 1, \|X_-^T\lambda_-\|_2\leq 1.
\end{equation}
The dual problem \eqref{maxmargin:dual_general} for the general loss, including hinge loss, is equivalent to 
\begin{equation}\label{dual:ortho_general}
    \max g(\lambda_+)+g(\lambda_-) \text{ s.t. } \lambda_+\geq0, \lambda_-\geq 0, \|X_+^T\lambda_+\|_2\leq \beta, \|X_-^T\lambda_-\|_2\leq \beta.
\end{equation}
Therefore, for arbitrary $\epsilon>0$, there exist polynomial-time algorithms to solve \texttt{Approx-Dual}, \texttt{Approx-Dual-Hinge} and \texttt{Approx-Dual-General}.
\end{theorem}

For orthogonal separable datasets, we note that the equivalent form \eqref{dual:ortho} of the dual problem can be separated into two optimization problems using the data with positive labels and the data with negative labels respectively.

\begin{theorem}\label{thm:ortho_primal}
Suppose that $(X,y)$ is orthogonal separable. Consider the following optimization problem 
\begin{equation}\label{primal:ortho}
    \min_{u_+,u_-\in\mbR^d} \|u_+\|_2+\|u_-\|_2 \text{ s.t. } X_+u_+\geq \bone,X_-u_-\geq\bone. 
\end{equation}
For a general loss consider the optimization problem
\begin{equation}\label{primal:ortho_general}
    \min_{u_+,u_-\in\mbR^d} \ell(X_+u_+)+\ell(X_-u_-)+\beta(\|u_+\|_2+\|u_-\|_2). 
\end{equation}
Let $u_+,u_-$ be the optimal solution to \eqref{primal:ortho}/\eqref{primal:ortho_general}. Let $W_1=\bmbm{\frac{u_+}{\sqrt{\|u_+\|_2}}&\frac{u_-}{\sqrt{\|u_-\|_2}}}$ and $w_2=\bmbm{\sqrt{\|u_+\|_2}&-\sqrt{\|u_-\|_2}}$. Then, $(W_1,w_2)$ is the optimal solution to \eqref{maxmargin:primal}/\eqref{maxmargin:primal_hinge}/\eqref{maxmargin:primal_general} as long as $m\geq 2$. Hence, for arbitrary $\epsilon>0$, there exist polynomial-time algorithms to solve \texttt{Approx-Primal}, \texttt{Approx-Primal-Hinge} and \texttt{Approx-Primal-General}.
\end{theorem}
From the constraint that $\diag(y)\lambda\geq 0$, we have $\lambda_+\geq 0$ and $\lambda_-\geq 0$. For orthogonal separable dataset, the maximin problems
$$
\max_{b_+\in[0,1]^{n_+}}\min_{b_-\in[0,1]^{n_-}}\|X_+^T\diag(\lambda_+)b_+-X_-^T\diag(\lambda_-)b_-\|_2
$$
and
$$
\max_{b_-\in[0,1]^{n_-}}\min_{b_+\in[0,1]^{n_+}}\|X_+^T\diag(\lambda_+)b_+-X_-^T\diag(\lambda_-)b_-\|_2
$$
can be solved in closed-form. 
\begin{proposition}\label{prop:ortho}
    Suppose that $(X,y)$ is orthogonal separable. Then, we have
    \begin{equation}
    \begin{aligned}
            &\max_{b_+\in[0,1]^{n_+}}\min_{b_-\in[0,1]^{n_-}}\|X_+^T\diag(\lambda_+)b_++X_-^T\diag(\lambda_-)b_-\|_2 = \|X_+\lambda_+\|_2,\\
        &\max_{b_-\in[0,1]^{n_-}}\min_{b_+\in[0,1]^{n_+}}\|X_+^T\diag(\lambda_+)b_++X_-^T\diag(\lambda_-)b_-\|_2 = \|X_-\lambda_-\|_2.
    \end{aligned}
    \end{equation}
\end{proposition}

Therefore, for orthogonal separable dataset, the dual problem reduces to \eqref{dual:ortho}. This is a convex optimization problem with cone constraints, which can be approximated in polynomial-time. 

\begin{lemma}\label{lem:ortho_sol}
The dual problem of \eqref{dual:ortho} is given by \eqref{primal:ortho}. The dual problem of \eqref{dual:ortho_general} is \eqref{primal:ortho_general}. Moreover, the optimal solution $(u_+,u_-)$ of \eqref{primal:ortho} or \eqref{primal:ortho_general} satisfies that $X_-u_+\leq 0$ and $X_+u_-\leq 0$. 
\end{lemma}

Therefore, by letting $W_1=\bmbm{\frac{u_+}{\sqrt{\|u_+\|_2}}&\frac{u_-}{\sqrt{\|u_-\|_2}}}$ and $w_2=\bmbm{\sqrt{\|u_+\|_2}&-\sqrt{\|u_-\|_2}}$, we have
\begin{equation}
    (X_+W_1)_+w_2=X_+u_+, (X_-W_1)_+w_2=-X_-u_-.
\end{equation}
We can verify that $(W_1,w_2)$ satisfies the constraints in \eqref{maxmargin:primal}. Moreover, we note that
\begin{equation}
    \frac{1}{2}\pp{\|W_1\|_F^2+\|w_2\|_2^2}=\|u_+\|_2+\|u_-\|_2=D.
\end{equation}
As $P_\text{non-cvx}\geq D$, this implies that $P_\text{non-cvx}=D$ and $(W_1,w_2)$ is the optimal solution to \eqref{maxmargin:primal}. The proofs for the hinge loss and the general loss are analogous.

\section{Approximation with relative error $\epsilon=(1-\sqrt{2/\pi})$ under negative correlation}\label{sec:neg}
We say a dataset $(X,y)$ has negative correlation if $x_ix_j\leq 0$ for $y_i\neq y_j$. 
\begin{theorem}\label{thm:neg_cor_dual}
Suppose that $(X,y)$ has negative correlation. With relative error $\epsilon=1-\sqrt{2/\pi}$, there exists a polynomial-time algorithm to solve \texttt{Approx-Dual} and \texttt{Approx-Dual-Hinge}.  For $g(\lambda)$ satisfying $g(a\lambda)\geq aC g(\lambda),\forall a\in(0,1]$ with parameter $C\in(0,1]$, with $\epsilon=1-C\sqrt{2/\pi}$, there exist a polynomial-time algorithm to solve \texttt{Approx-Dual-General}.
\end{theorem}
According to Lemma \ref{lem:maximin}, we can reduce the dual constraints into constraints on two maximin problems. By utilizing the property of the dataset, we can further simply the constraints by the following proposition. 

\begin{proposition}\label{prop:neg}
    Suppose that $(X,y)$ has negative correlation. Then, we have
    \begin{equation}
    \begin{aligned}
            &\max_{b_+\in[0,1]^{n_+}}\min_{b_-\in[0,1]^{n_-}}\|X_+^T\diag(\lambda_+)b_++X_-^T\diag(\lambda_-)b_-\|_2 
            = \max_{b_+\in[0,1]^{n_+}}\|X_+^T\diag(\lambda_+)b_+\|_2,\\
        &\max_{b_-\in[0,1]^{n_-}}\min_{b_+\in[0,1]^{n_+}}\|X_+^T\diag(\lambda_+)b_++X_-^T\diag(\lambda_-)b_-\|_2 
        = \max_{b_-\in[0,1]^{n_-}}\|X_-^T\diag(\lambda_-)b_-\|_2.
    \end{aligned}
    \end{equation}
\end{proposition}
\begin{proof}
The proof follows from the proof of Proposition \ref{prop:ortho}.
\end{proof}
Therefore, we can rewrite the dual problem as 
\begin{equation}
\begin{aligned}
    \max \;&\lambda_+^T\bone+\lambda_-\bone 
    \text{ s.t. } &\lambda_+\geq0, \lambda_-\geq0, 
    &\max_{b_+\in[0,1]^{n_+}}\|X_+^T\diag(\lambda_+)b_+\|_2\leq 1,
    &\max_{b_-\in[0,1]^{n_-}}\|X_-^T\diag(\lambda_-)b_-\|_2\leq 1.
\end{aligned}
\end{equation}
We can solve the dual problem by solving the following two problems separately:
\begin{equation}\label{dual:p1}
\begin{aligned}
    \max \;&\lambda_+^T\bone   \text{ s.t. } &\lambda_+\geq0, \max_{b_+\in[0,1]^{n_+}}\|X_+^T\diag(\lambda_+)b_+\|_2\leq 1.
\end{aligned}
\end{equation}
\begin{equation}\label{dual:p2}
\begin{aligned}
    \max \;&\lambda_-^T\bone    \text{ s.t. } &\lambda_-\geq0, \max_{b_-\in[0,1]^{n_+}}\|X_-^T\diag(\lambda_-)b_-\|_2\leq 1.
\end{aligned}
\end{equation}
To show that \texttt{Approx-Dual} can be solved in polynomial-time with relative error $\epsilon=\pi/2-1$, it is sufficient to show that the following problem can be approximated polynomial-time with relative error $\epsilon=\pi/2-1$:
\begin{equation}\label{dual:reduce}
\begin{aligned}
    \max \;&\lambda^T\bone,   \text{ s.t. } &\lambda\geq0, \max_{b\in[0,1]^{n}}\|X^T\diag(\lambda)b\|_2\leq 1.
\end{aligned}
\end{equation}
For fixed $\lambda\in\mbR^n$, let $c_1(\lambda)$ be the optimal value of 
$$
\frac{1}{4}\max_{z\in\{-1,1\}^{n+1}} z^T\bmbm{I\\\bone^T}\diag( \lambda)XX^T\diag(  \lambda)\bmbm{I&\bone} z.
$$
According to \eqref{equ:connect_maxcut}, we can rewrite the dual constraint as $c_1(\lambda)\leq 1$. Consider the convex program with optimal value $c_2(\lambda)$:
\begin{equation}
    c_2(\lambda)=\frac{1}{4}\max_{Z\in\mbS^{n+1}} \tr\pp{Z\bmbm{I\\\bone^T}\diag( \lambda)XX^T\diag(\lambda)\bmbm{I&\bone}}, \text{ s.t. } \diag(Z)=1, Z\succeq 0.
\end{equation}
\begin{lemma}\label{lem:sdp_bnd}
Suppose that $ Q\succeq 0$. Consider the following optimization problem
\begin{equation}
    \textbf{OPT}=\max_{z\in\{-1,1\}^{n+1}} z^T Qz.
\end{equation}
The SDP relaxation is given by
\begin{equation}
     \textbf{SDP}=\max_{Z\in\mbS^n} \tr(Z Q) \text{ s.t. }\diag(Z)=\bone, Z\succeq 0.
\end{equation}
Let $Z_0\succeq 0$ be an arbitrary covariance matrix with $\diag(Z_0)=1$. Then, we have
\begin{equation}
    \mbE_{z=\text{sign}(r),r\sim\mcN(0,Z_0)}[z^TQz]\geq \frac{2}{\pi} \textbf{SDP}\geq \frac{2}{\pi} \textbf{OPT}.
\end{equation}
\end{lemma}
From Lemma \ref{lem:sdp_bnd}, we note that
\begin{equation}\label{equ:bnd}
   \frac{2}{\pi} c_2(\lambda)\leq c_1(\lambda)\leq c_2(\lambda).
\end{equation}
Therefore, consider the following optimization problem. 
\begin{equation}\label{dual:sdp}
\begin{aligned}
    \max\;&\lambda^T\bone
    \text{ s.t. }&\max_{Z\succeq 0, \diag(Z)=\bone} \frac{1}{4}\tr(Z\bmbm{I\\\bone^T}\diag(\lambda)XX^T\diag(\lambda)\bmbm{I&\bone})\leq 1,
    &\lambda\geq 0.
\end{aligned}
\end{equation}
Suppose that $\tilde \lambda$ is the optimal solution to \eqref{dual:sdp}. We can show that $\tilde \lambda$ is an approximate solution to the dual problem, i.e., 
$$
 \tilde \lambda^Ty\geq D/\sqrt{\pi/2},;\max_{u:\|u\|_2\leq 1}| \tilde \lambda^T(Xu)_+|\leq 1.
$$
Let $\lambda^*$ be the optimal solution to \eqref{maxmargin:dual}. We note that
$$
c_2(\sqrt{2/\pi} \lambda^*)=\frac{2}{\pi} c_2(\lambda^*)\leq c_1(\lambda^*)\leq 1.
$$
This implies that $\sqrt{2/\pi} \lambda^*$ is feasible to \eqref{dual:sdp}. Thus, we have
\begin{equation}
    D/\sqrt{\pi/2}=(\sqrt{2/\pi} \lambda^*)^Ty\leq  \tilde \lambda^Ty.
\end{equation}
On the other hand, we have $c_1(\tilde \lambda)\leq c_2(\tilde \lambda)\leq 1$, which implies that $\max_{u:\|u\|_2\leq 1}| \tilde \lambda^T(Xu)_+|\leq 1$. 

For the hinge loss, we consider the following problem:
\begin{equation}\label{dual:sdp_hinge}
\begin{aligned}
    \max\;&\lambda^T\bone
    \text{ s.t. }&\max_{Z\succeq 0, \diag(Z)=\bone} \frac{1}{4}\tr(Z\bmbm{I\\\bone^T}\diag(\lambda)XX^T\diag(\lambda)\bmbm{I&\bone})\leq \beta^2,
    &0\leq \lambda\leq 1.
\end{aligned}
\end{equation}
Let $\tilde \lambda$ be the optimal solution to the above problem. It is sufficient to show that $\tilde \lambda$ is an approximate solution to the dual problem, i.e., 
$$
 \tilde \lambda^Ty\geq D^\text{hinge}/\sqrt{\pi/2},;\max_{u:\|u\|_2\leq 1}| \tilde \lambda^T(Xu)_+|\leq \beta, 0\leq \tilde \lambda\leq1.
$$
Let $\lambda^*$ be the optimal solution to \eqref{maxmargin:dual_general}. Similarly, we can show that $\sqrt{2/\pi} \lambda^*$ is feasible to \eqref{dual:sdp_hinge}. Thus, we have
\begin{equation}
    D^\text{hinge}/\sqrt{\pi/2}=(\sqrt{2/\pi} \lambda^*)^Ty\leq  \tilde \lambda^Ty.
\end{equation}
On the other hand, we have $c_1(\tilde \lambda)\leq c_2(\tilde \lambda)\leq 1$, which implies that $\max_{u:\|u\|_2\leq 1}| \tilde \lambda^T(Xu)_+|\leq 1$. 

For the general loss, we consider the following problem:
\begin{equation}\label{dual:sdp_gen}
\begin{aligned}
    \max\;&g(\lambda)
    \text{ s.t. }&\max_{Z\succeq 0, \diag(Z)=\bone} \frac{1}{4}\tr(Z\bmbm{I\\\bone^T}\diag(\lambda)XX^T\diag(\lambda)\bmbm{I&\bone})\leq \beta^2,
    &\lambda\geq 1.
\end{aligned}
\end{equation}
Let $\tilde \lambda$ be the optimal solution to the above problem. It is sufficient to show that $\tilde \lambda$ is an approximate solution to the dual problem, i.e., 
$$
 g(\tilde \lambda) \geq C D^\text{gen}/\sqrt{\pi/2}, \max_{u:\|u\|_2\leq 1}| \tilde \lambda^T(Xu)_+|\leq \beta, \tilde \lambda\geq0.
$$
Let $\lambda^*$ be the optimal solution to \eqref{maxmargin:dual_general}. Similarly, we can show that $\sqrt{2/\pi} \lambda^*$ is feasible to \eqref{dual:sdp_hinge}. Thus, we have
\begin{equation}
    C D^\text{gen}/\sqrt{\pi/2}\leq g(\sqrt{2/\pi} \lambda^*)\leq g(\tilde \lambda).
\end{equation}

For \eqref{dual:sdp} or \eqref{dual:sdp_hinge} or \eqref{dual:sdp_gen}, we can use the ellipsoid method in \cite{boyd1991linear} to solve it. Denote $S=\{\lambda:c_2(\lambda)\leq 1\}\subseteq \mbR^n$, which is a convex set. We describe the definition of Approximate Separation Oracle introduced in \cite{bhattiprolu2021framework} as follows.
\begin{definition}
    For $\alpha\geq 1$, an $\alpha$-approximate separation oracle for an $\alpha$-separable body $B\subseteq R^n$ is an oracle that on any input point $x\in\mbR^n$, either correctly outputs `Inside' when $x\in \alpha\cdot B$ or outputs a hyperplane separating $x$ from $B$.
\end{definition}
In the following proposition, we can construct an $1$-approximate separation oracle for $S$. 
\begin{proposition}\label{prop:oracle}
For any $\lambda_0\in\mbR^n$, we compute $c_2(\lambda_0)$ by solving the SDP. We generate the $1$-approximate separation oracle for $S$ as follows. If $c_2(\lambda_0)\leq 1$, outputs `Inside'. Otherwise, outputs $H=\{\lambda|c_2'(\lambda_0)(\lambda-\lambda_0)=0\}$. 
\end{proposition}
\begin{proposition}\label{prop:poly}
There exists a polynomial-time algorithm in time $\text{poly}(n,\log(1/\epsilon))$ returns $\lambda\in S$ such that
\begin{equation}
    \lambda^Ty\geq d^*-\epsilon,
\end{equation}
where $d^*$ is the optimal value of \eqref{dual:sdp}. 
\end{proposition}
\begin{proof}
This is a direct corollary of Proposition 3.6 in \cite{bhattiprolu2021framework}. 
\end{proof}
This implies that the ellipsoid method produces an $1-\sqrt{2/\pi}$ approximation of the dual problem in polynomial-time. Similarly, we can solve \eqref{dual:sdp_hinge} or \eqref{dual:sdp_gen} in polynomial-time. This completes the proof for Theorem \ref{thm:neg_cor_dual}.

\subsection{From Approx-dual to Approx-primal}
\begin{theorem}\label{thm:neg_cor_primal}
Let $\epsilon_0>0$ and $\delta>0$. Suppose that $(X,y)$ has negative correlation. With relative error $\epsilon=\sqrt{\pi/2}(1+\epsilon_0)-1$, there exists a polynomial-time algorithm (polynomial in $\epsilon_0^{-1}$ and $\log(\delta^{-1})$) to solve \texttt{Approx-Primal} and \texttt{Approx-Primal-Hinge} with probability at least $1-\delta$. For $g(\lambda)$ satisfying $g(a\lambda)\geq a C g(\lambda),\forall a\in(0,1]$ with parameter $C\in(0,1]$, with $\epsilon=C^{-1}\sqrt{\pi/2}-1$, there exists a polynomial-time algorithm to solve the \texttt{Approx-Primal-General} problem with probability at least $1-\delta$.
\end{theorem}

To prove Theorem \ref{thm:neg_cor_primal}, we first prove for the special case when $y=\bone$. Suppose that $\tilde \lambda$ is the optimal solution to \eqref{dual:sdp}. We also denote $\tilde Z$ as an optimal solution to the SDP
\begin{equation}\label{equ:sdp}
    \max_{Z\in\mbS^n} \tr(Z\bmbm{I\\\bone^T}\diag(\tilde \lambda)XX^T\diag( \tilde \lambda)\bmbm{I&\bone}) \text{ s.t. }\diag(Z)=\bone, Z\succeq 0.
\end{equation}
Let $k$ be an integer number. We randomly draw i.i.d. samples $r^1,r^2,\dots, r^k\sim \mcN(0,\tilde Z)$ and we let $z^i=\text{sign}(r^i)$. For $i\in[k]$, we define $b^i\in\{0,1\}^n$ by
$$
b^i_j = (z^i_jz^i_{n+1}+1)/2, \forall j \in [n].
$$
Then, we have
\begin{equation}
\begin{aligned}
    &\frac{1}{4}(z^i)^T\bmbm{I\\\bone^T}\diag(\tilde\lambda)XX^T\diag(\tilde \lambda)\bmbm{I&\bone}z^i 
    = &\frac{1}{4}\norm{X^T\diag(\lambda)(z^i_{1:n}+\bone z^i_{n+1})}_2^2 = \norm{X^T\diag(\lambda)b^i}_2^2.
\end{aligned}
\end{equation}
The following lemma indicates that for each previously constructed $b^i\in\{0,1\}^n$, it corresponds to the $0/1$ pattern of a hyperplane arrangement.  
\begin{lemma}\label{lem:valid}
Suppose that $\tilde Z$ is an optimal solution to \eqref{equ:sdp}. Let $r\sim \mcN(0,\tilde Z)$ and $z=\text{sign}(r)$. Define $b\in\{0,1\}^n$ by
$$
b_j = (z_jz_{n+1}+1)/2, \forall j\in[n].
$$
Then, there exists $w\in\mbR^d$ such that $b=\mbI(Xw\geq 0)$. 
\end{lemma}

Consider the following optimization problem:
\begin{equation}\label{primal:sub}
\begin{aligned}
    p=\min \; \sum_{i=1}^k \|u_i\|_2
    \text{ s.t. } \sum_{i=1}^k \diag(b^i)Xu_i\geq \bone, i\in[k]. 
\end{aligned}
\end{equation}
We can show that by choosing $k=O(\log(n)C_1^2(C_2\epsilon)^{-2}\log(\delta))$, with probability at least $1-\delta$, we have
\begin{equation}\label{p:bound}
    P_\text{non-cvx}\leq p\leq \sqrt{(1+\epsilon)\pi/2}P_\text{non-cvx}.
\end{equation}

From Lemma \ref{lem:valid}, there exist $h_1,\dots,h_k\in\mbR^n$ such that $b^i=\mbI(Xw_i\geq 0)$ for $i\in[k]$. Thus, let $\{u_i\}_{i=1}^k$ be the optimal solution to \eqref{primal:sub}. Denote $H=\bmbm{h_1&\dots&h_k}$, $W_1=\bmbm{\frac{u_1}{\sqrt{\|u_1\|_2}}&\dots&\frac{u_k}{\sqrt{\|u_k\|_2}}}$, $w_2=\bmbm{\sqrt{\|u_1\|_2}\\\vdots\\\sqrt{\|u_k\|_2}}$ and $\Theta=(H,W_1,w_2)$, we have
\begin{equation}
    \sum_{i=1}^k \diag(b^i)Xu_i = f^\text{gReLU}(X;\Theta), \sum_{i=1}^k\|u_i\|_2=R(\Theta). 
\end{equation}

Now, we proceed to prove \eqref{p:bound}. The dual problem of the above problem is given by
\begin{equation}\label{dual:sub}
    \max \lambda^T\bone\; \text{ s.t. } \lambda\geq 0, \max_{i\in[k]}\|X^T\diag(\lambda)b^i\|_2\leq 1
\end{equation}
The optimal value of \eqref{dual:sub} is lower-bounded by $D=P_\text{non-cvx}$.  Let $Z_k=\frac{2}{\pi}\frac{1}{k}\sum_{i=1}^kz^i(z^i)^T$.
The optimal value of \eqref{dual:sub} is also upper bounded by
\begin{equation}\label{dual:sub2}
\begin{aligned}
    d_1=&\max \lambda^T\bone\; \text{ s.t. } \lambda\geq 0, \sum_{i=1}^k\frac{1}{k}\|X^T\diag(\lambda)b^i\|_2^2\leq 1\\
    =&\max \lambda^T\bone\; \text{ s.t. } \lambda\geq 0,\frac{1}{4}\tr(Z_k\bmbm{\diag(\lambda)X\\\bone^T\diag(\lambda)X}\bmbm{X^T\diag( \lambda)&X^T\diag(\lambda)\bone})\leq \frac{2}{\pi}\\
\end{aligned}
\end{equation}

\begin{lemma}\label{lem:prob_bnd}
Let $\delta>0$ and $\epsilon>0$. By taking 
$$
k=O(\log(n)C_1^2(C_2\epsilon)^{-2}\log(1/\delta)),
$$
with probability at least $1-\delta$, we have that
\begin{equation}\label{inequ:1}
\begin{aligned}
    \tr(Z_k\bmbm{I\\\bone^T}\diag(\lambda)XX^T\diag( \lambda)\bmbm{I&\bone})
    \leq(1+\epsilon) \tr(\tilde Z\bmbm{I\\\bone^T}\diag(\lambda)XX^T\diag( \lambda)\bmbm{I&\bone})
\end{aligned}
\end{equation}
holds for all $\lambda\geq 0$. 
Here $C_1$ is the optimal value of 
$$
\min_{\lambda:\lambda^T\bone=1,\lambda\geq 0} \tr(\tilde Z\bmbm{I\\\bone^T}\diag(\lambda)XX^T\diag( \lambda)\bmbm{I&\bone})
$$
and $C_2=\max_{i\in[n]}\|x_i\|_2^2$. 
\end{lemma}
Consider the following problem
\begin{equation}\label{dual:sdp_sub}
\begin{aligned}
    \max\;&\lambda^T\bone
    \text{ s.t. }&\frac{1}{4}\tr(\tilde Z\bmbm{I\\\bone^T}\diag(\lambda)XX^T\diag( \lambda)\bmbm{I&\bone})\leq 1,
    &\lambda\geq 0.
\end{aligned}
\end{equation}
Let $d_2$ be its optimal value and let $\lambda_2$ be its optimal solution. Then, by properly choosing $k$ in Lemma \ref{lem:prob_bnd}, with probability at least $1-\delta$,  $\sqrt{\frac{2}{\pi(1+\epsilon)}}\lambda_2$ is feasible to \eqref{dual:sub2}. This implies that 
\begin{equation}
    d_1\leq \sqrt{\frac{\pi(1+\epsilon)}{2}}d_2.
\end{equation}
Let $d_3$ be the optimal value of \eqref{dual:sdp}. The following lemma shows that $d_2=d_3$. 
\begin{lemma}\label{lem:sdp_minmax}
The problem \eqref{dual:sdp} is equivalent to \eqref{dual:sdp_sub}.
\end{lemma}
On the other hand, from Lemma \ref{lem:sdp_bnd}, we note that
\begin{equation}
     d_3\leq D.
\end{equation}
Therefore, the optimal value of \eqref{primal:sub} is upper bounded by
\begin{equation}
    d_1\leq \sqrt{\frac{\pi(1+\epsilon)}{2}}d_2=\sqrt{\frac{\pi(1+\epsilon)}{2}}d_3\leq \sqrt{\frac{\pi(1+\epsilon)}{2}}D=\sqrt{\frac{\pi(1+\epsilon)}{2}} P_\text{non-cvx}.
\end{equation}
This completes the proof. 

For general dataset $(X,y)$ with negative correlation, for $X_+$ and $X_-$, we can draw $b^1_+,\dots,b^{k_+}_+$ and $b^1_-,\dots,b^{k_-}_-$ and consider the following optimization problems
\begin{equation}\label{equ:pp}
\begin{aligned}
    p_+=\min \; \sum_{i=1}^{k_+} \|u_i\|_2
    \text{ s.t. } \sum_{i=1}^{k_+} \diag(b^i_+)X_+u_i\geq \bone, i\in[k]. 
\end{aligned}
\end{equation}
\begin{equation}\label{equ:pm}
\begin{aligned}
    p_-=\min \; \sum_{i=1}^{k_-} \|v_i\|_2
    \text{ s.t. } \sum_{i=1}^k \diag(b^i_-)X_-v_i\geq \bone, i\in[k]. 
\end{aligned}
\end{equation}

From the previous proof, by choosing 
$$
k_+=O(\log(n_+)(\epsilon)^{-2}\log(\delta))\text{ and }k_-=O(\log(n_-)(\epsilon)^{-2}\log(\delta)),
$$
with probability at least $1-\delta$, we have
\begin{equation}\label{p:bound2}
    P_\text{non-cvx}\leq p_++p_-\leq \sqrt{(1+\epsilon)\pi/2}P_\text{non-cvx}.
\end{equation}
From Lemma \ref{lem:valid}, there exist $h_1^+,\dots,h_{k_+}^+,h_1^-,\dots,h_{k_-}^-\in\mbR^n$ such that $b^i_+=\mbI(X_+h_i^+\geq 0)$ for $i\in[k_+]$ and $b^i_-=\mbI(X_-h_i^-\geq 0)$. For simplicity, we assume that $X=\bmbm{X_+\\X_-}$. From the proof of Lemma \ref{lem:valid}, we have $h_1^+,\dots,h_{k_+}^+\in\text{range}(X_+)$ and $h_1^-,\dots,h_{k_-}^-\in\text{range}(X_-)$. As the dataset has negative correlation, this implies that 
$$
X_-h_i^+\leq 0, i\in[k_+], X_+h_i^-\leq 0, i\in[k_-]. 
$$
Let $\{u_i\}_{i=1}^{k_+}$ and $\{v_i\}_{i=1}^{k_-}$ be the optimal solutions to \eqref{equ:pp} and \eqref{equ:pm} respectively. Denote 
$$
\begin{aligned}
H=&\bmbm{h_1^+&\dots&h_{k_+}^+&h_1^-&\dots&h_{k_-}^-}, \\
W_1=&\bmbm{\frac{u_1}{\sqrt{\|u_1\|_2}}&\dots&\frac{u_{k_+}}{\sqrt{\|u_{k_+}\|_2}}&\frac{v_1}{\sqrt{\|v_1\|_2}}&\dots&\frac{v_{k_-}}{\sqrt{\|v_{k_-}\|_2}}}, \\
w_2=&\bmbm{\sqrt{\|u_1\|_2}&\dots&\sqrt{\|u_{k_+}\|_2}&-\sqrt{\|v_1\|_2}&\dots&-\sqrt{\|v_{k_-}\|_2}}^T.
\end{aligned}
$$
and $\Theta=(H,W_1,w_2)$. Then, we note that
\begin{equation}
    f^\text{gReLU}(X;\Theta) = \bmbm{\sum_{i=1}^{k_+} \diag(b^i_+)X_+u_i\\-\sum_{i=1}^k \diag(b^i_-)X_-v_i}, R(\Theta)=p_++p_-.
\end{equation}
Therefore, we completes the proof for Theorem \ref{thm:neg_cor_primal} for the max-margin problem.

For the training problem with hinge loss, it is sufficient to prove for the special case when $y=\bone$. If we have
$$
\beta\geq \max_{u:\|u\|_2\leq 1} |y^T(Xu)_+|,
$$
then, $y$ is feasible to $D^\text{hinge}$ and we have $D^\text{hinge}=n$. In this case, we have $P_\text{non-cvx}^\text{hinge}=n$. It can be achieved when $W_1=0$ and $w_2=0$.

Consider the case when $\beta<\max_{u:\|u\|_2\leq 1} |y^T(Xu)_+|$. Suppose that $\tilde \lambda$ is the optimal solution to \eqref{dual:sdp_hinge}. We also denote $\tilde Z$ as an optimal solution to the SDP \eqref{equ:sdp}. We randomly sample $b^1,\dots,b^k$ as previously we did in the case of max-margin problem. Consider the following optimization problem:
\begin{equation}\label{primal:sub_hinge}
\begin{aligned}
    p=\min \; \bone^T(\sum_{i=1}^k \diag(b^i)Xu_i-\bone)_+ +\beta \sum_{i=1}^k \|u_i\|_2.
\end{aligned}
\end{equation}
We can show that by choosing $k=O(\log(n)C_1^2(C_2\epsilon)^{-2}\log(\delta))$, with probability at least $1-\delta$, we have
\begin{equation}
    P_\text{non-cvx}^\text{hinge}\leq p\leq \sqrt{(1+\epsilon)\pi/2}P_\text{non-cvx}^\text{hinge}.
\end{equation}
The proof is analogous to the max-margin case, except that we need a tweak for Lemma \ref{lem:sdp_minmax}. It is sufficient to show that when $\beta<\max_{u:\|u\|_2\leq 1} |y^T(Xu)_+|$, the problem \eqref{dual:sdp_hinge} is equivalent to
\begin{equation}\label{dual:sdp_sub_hinge}
\begin{aligned}
    \max\;&\lambda^T\bone
    \text{ s.t. }&\frac{1}{4}\tr(\tilde Z\bmbm{I\\\bone^T}\diag(\lambda)XX^T\diag( \lambda)\bmbm{I&\bone})\leq \beta^2,0\leq \lambda\leq \bone.
\end{aligned}
\end{equation}
The proof is given in Appendix \ref{proof:lem:sdp_minmax}.

For the general loss, analogously it is sufficient to prove for the special case when $y=\bone$. Suppose that $\tilde \lambda$ is the optimal solution to \eqref{dual:sdp_hinge}. We also denote $\tilde Z$ as an optimal solution to the SDP \eqref{equ:sdp}. We randomly sample $b^1,\dots,b^k$ as previously we did in the case of max-margin problem. Consider the following optimization problem:
\begin{equation}\label{primal:sub_general}
\begin{aligned}
    p=\min \; \ell(\sum_{i=1}^k \diag(b^i)Xu_i)+ \beta  \sum_{i=1}^k\|u_i\|_2.
\end{aligned}
\end{equation}
Similarly, we can show that by choosing $k=O(\log(n)C_1^2(C_2\epsilon)^{-2}\log(\delta))$, with probability at least $1-\delta$, we have
\begin{equation}
    P_\text{non-cvx}^\text{gen}\leq p\leq C^{-1}\sqrt{(1+\epsilon)\pi/2}P_\text{non-cvx}^\text{gen}.
\end{equation}

\section{Approximation bound for general datasets}\label{sec:geo}
For a general dataset $(X,y)$, the main difficulty is to handle the constraints on the optimal value of two maximin problems. In order to analyze this, we introduce the following geometric ratio.
\begin{definition}[Geometric ratio]\label{def:ratio}
Let $\lambda^*$ is the optimal solution to the dual problem $D$. We define the geometric ratio $c^*=\frac{\max_{b_+\in\{0,1\}^{n_+}} \|X_+^T\diag(\lambda^*_+)b_+\|_2}{\max_{b_-\in\{0,1\}^{n_-}}\|X_-^T\diag(\lambda^*_-)b_-\|_2}$. To be specific, $c^*$ is the ratio between the Haussdorff distances between zero and two zonotopes $K_+$ and $K_-$ defined in \eqref{equ:zono} with $\lambda=\lambda^*$ . 
\end{definition}
Next, we present our main approximation result based on the geometric ratio.
\begin{theorem}\label{thm:geo_bnd}
Let $c\in(0,1)$. Suppose that $c\leq \min\{c^*,(c^*)^{-1}\}$, where $c^*$ is a geometric ratio defined in Definition \ref{def:ratio}. Then, with relative error $\epsilon=1-(1-c)\sqrt{2/\pi}$, there exists a polynomial-time algorithm to solve \texttt{Approx-Dual} and \texttt{Approx-Dual-Hinge}. Moreover, we can find a value $p$ to solve \texttt{Approx-Primal} and \texttt{Approx-Primal-Hinge} in polynomial-time with relative error $\epsilon=(1-c)^{-1}\sqrt{\pi/2} -1$. For $g(\lambda)$ satisfying $g(a\lambda)\geq a C g(\lambda),\forall a\in(0,1]$ with parameter $C\in(0,1]$, with $\epsilon=1-C^{-1}\sqrt{2/\pi}$, there exist a polynomial-time algorithm to solve \texttt{Approx-Dual-General}. We can find a value $p$ to solve \texttt{Approx-Primal-General} in polynomial-time with $\epsilon=C^{-1}\sqrt{\pi/2}-1$,
\end{theorem}

According to Proposition \ref{prop:hauss_dist}, from the triangle inequality, we obtain
\begin{equation}\label{equ:tri}
    H(K_+,K_-)\geq H(K_+,0)-H(K_-,0), \; H(K_+,K_-)\geq H(K_-,0)-H(K_+,0).
\end{equation}
Here, we note that
\begin{equation}
    H(K_+,0)=\max_{b_+\in[0,1]^{n_+}}\|X_+^T\diag(\lambda_+)b_+\|_2,\\
    H(K_-,0)=\max_{b_-\in[0,1]^{n_-}}\|X_-^T\diag(\lambda_-)b_-\|_2,\\
\end{equation}

\begin{proposition}\label{prop:general}
Let $c\in(0,1)$. Consider the following problem 
\begin{equation}\label{equ:dc1}
\begin{aligned}
    D_c^1=&\max \lambda^Ty 
    \text{ s.t. } &\diag(y)\lambda \geq 0,
    &\max_{b_+\in\{0,1\}^{n_+}} \|X_+^T\diag(\lambda_+)u\|_2\leq 1 ,  
    &\max_{b_-\in\{0,1\}^{n_-}} \|X_-\diag(\lambda_-)v\|_2\leq c.
\end{aligned}
\end{equation}
\begin{equation}
\begin{aligned}
    D_c^2=&\max \lambda^Ty 
    \text{ s.t. } &\diag(y)\lambda \geq 0,
    &\max_{b_+\in\{0,1\}^{n_+}} \|X_+^T\diag(\lambda_+)u\|_2\leq c,  
    &\max_{b_-\in\{0,1\}^{n_-}} \|X_-\diag(\lambda_-)v\|_2\leq 1.
\end{aligned}
\end{equation}
Suppose that $c\leq \min\{c^*,(c^*)^{-1}\}$, where $c^*$ is a geometric constant determined by $X_+$ and $X_-$. Then, if $c^*>1$, we have the approximation ratio
\begin{equation}
\begin{aligned}
(1-c)D \leq D_c^1\leq D.
\end{aligned}
\end{equation}
If $c^*<1$, we have the approximation ratio
\begin{equation}
\begin{aligned}
(1-c)D \leq D_c^2\leq D
\end{aligned}
\end{equation}
\end{proposition}

Now we present the proof for Theorem \ref{thm:geo_bnd} for the max-margin problem. The proofs for the hinge loss and the max-margin loss are given in Appendix \ref{app:geo_bnd}.  Without the loss of generality, we can assume that $c^*>1$. We can rewrite \eqref{equ:dc1} into two independent optimization problems \eqref{dual:p1} and \eqref{dual:p2}. According to Proposition \ref{prop:poly}, there exists a polynomial-time-algorithm which returns $\lambda_+$ and $\lambda_-$ such that 
\begin{equation}
    \max_{b_+\in\{0,1\}^{n_+}} \|X_+^T\diag(\lambda_+)u\|_2\leq 1 ,  \max_{b_-\in\{0,1\}^{n_-}} \|X_-\diag(\lambda_-)v\|_2\leq c,
\end{equation}
and
$$
\bone^T\lambda_+-\bone^T\lambda_-\geq \sqrt{2/\pi} D_c.
$$
We note that
\begin{equation}
    \max_{b_+\in[0,1]^{n_+}}\min_{b_-\in[0,1]^{n_-}}\|X_+^T\diag( \lambda_+)b_++X_-^T\diag( \lambda_-)b_-\|_2\leq  \max_{b_+\in[0,1]^{n_+}}\|X_+^T\diag( \lambda_+)b_+\|_2\leq 1
\end{equation}
This implies that $\lambda$ is feasible for $D$. We also have
\begin{equation}
    \bone^T\lambda_+-\bone^T\lambda_-\geq \sqrt{2/\pi} D_c^1 \geq \sqrt{2/\pi}(1-c) D.
\end{equation}
This implies that we can find an approximate solution to the dual problem with relative error $1-\sqrt{2/\pi}(1-c)$. On the other hand, we note that $D=P_\text{non-cvx}$. Thus, by taking $p=\frac{\bone^T\lambda_+-\bone^T\lambda_-}{(1-c)\sqrt{2/\pi}}$
\begin{equation}
    P_\text{non-cvx}\leq p\leq (1-c)^{-1}\sqrt{\pi/2} P_\text{non-cvx}.
\end{equation}
This implies that we can find a value $p$ which solves \text{Approx-Primal} with relative error $\epsilon= (1-c)^{-1}\sqrt{\pi/2} -1$.

\section{Conclusion}
In this work, we explored the nuanced landscape of polynomial-time approximations to the training problem associated with 2-layer ReLU neural networks. Our results show that even for straightforward datasets with all positive labels, the problem of achieving an approximation with a relative error less than \( \sqrt{84/83}-1 \) is {NP-hard}. However, for structured datasets, such as those that are orthogonally separable, we can determine the exact solution in polynomial time. Moreover, for datasets that contain negative correlations across different classes, we can achieve a relative error of \( \sqrt{\pi/2}-1 \) in polynomial-time. For general datasets, the relative error of a polynomial-time-derived approximate solution is governed by a geometric factor inherent to the dataset. Our results also hold general loss functions under mild restrictions. Our results can be extended to deeper neural networks by considering the convex formulations.

\section{Acknowledgements}
This work was supported in part by the National Science Foundation (NSF) CAREER Award under Grant CCF-2236829, Grant ECCS-2037304 and Grant DMS-2134248; in part by the U.S. Army Research Office Early Career Award under Grant W911NF-21-1-0242; in part by the Stanford Precourt Institute; and in part by the ACCESS—AI Chip Center for Emerging Smart Systems through InnoHK, Hong Kong, SAR.

\newpage 
\appendix

\section{Proofs in Section \ref{sec:dual}}

\subsection{Proof of Lemma \ref{lem:maximin}}
\begin{proof}
We note that 
$$
\max_{u:\|u\|_2\leq 1} |\lambda^T(Xu)_+|=\max\{\max_{u:\|u\|_2\leq 1} \lambda^T(Xu)_+,\max_{u:\|u\|_2\leq 1} -\lambda^T(Xu)_+\}.
$$
For the first component, we can compute that
$$
\begin{aligned}
&\max_{u:\|u\|_2\leq 1} \lambda^T(Xu)_+\\
=&\max_{u:\|u\|_2\leq 1} \lambda_+^T(X_+u)_+-\lambda_-^T(X_-u)_+\\
=&\max_{u:\|u\|_2\leq 1}\max_{b_+\in[0,1]^{n_+}}\min_{b_-\in[0,1]^{n_-}}\lambda_+^T\diag(b_+)X_+u-\lambda_-^T\diag(b_-)X_-u\\
=&\max_{b_+\in[0,1]^{n_+}}\max_{u:\|u\|_2\leq 1}\min_{b_-\in[0,1]^{n_-}}\lambda_+^T\diag(b_+)X_+u-\lambda_-^T\diag(b_-)X_-u\\
\stepa{=}&\max_{b_+\in[0,1]^{n_+}}\min_{b_-\in[0,1]^{n_-}}\max_{u:\|u\|_2\leq 1}\lambda_+^T\diag(b_+)X_+u-\lambda_-^T\diag(b_-)X_-u\\
=&\max_{b_+\in[0,1]^{n_+}}\min_{b_-\in[0,1]^{n_-}}\|X_+^T\diag(b_+)\lambda_+-X_-^T\diag(b_-)\lambda_-\|_2\\
=&\max_{b_+\in[0,1]^{n_+}}\min_{b_-\in[0,1]^{n_-}}\|X_+^T\diag(\lambda_+)b_+-X_-^T\diag(\lambda_-)b_-\|_2
\end{aligned}
$$
Here step (a) utilizes that for fixed $b_+\in[0,1]^{n_+}$, we have
$$
\begin{aligned}
    &\max_{u:\|u\|_2\leq 1}\min_{b_-\in[0,1]^{n_-}}\lambda_+^T\diag(b_+)X_+u+\lambda_-^T\diag(b_-)X_-u\\
=&\min_{b_-\in[0,1]^{n_-}}\max_{u:\|u\|_2\leq 1}\lambda_+^T\diag(b_+)X_+u+\lambda_-^T\diag(b_-)X_-u.
\end{aligned}
$$
On the other hand, we note that
$$
\begin{aligned}
&\max_{u:\|u\|_2\leq 1} (-\lambda)^T(Xu)_+\\
=&\max_{u:\|u\|_2\leq 1} -\lambda_+^T(X_+u)_++\lambda_-^T(X_-u)_+\\
=&\max_{u:\|u\|_2\leq 1}\max_{b_-\in[0,1]^{n_-}}\min_{b_+\in[0,1]^{n_+}}-\lambda_+^T\diag(b_+)X_+u+\lambda_-^T\diag(b_-)X_-u\\
=&\max_{b_-\in[0,1]^{n_-}}\max_{u:\|u\|_2\leq 1}\min_{b_+\in[0,1]^{n_+}}-\lambda_+^T\diag(b_+)X_+u+\lambda_-^T\diag(b_-)X_-u\\
=&\max_{b_-\in[0,1]^{n_-}}\min_{b_+\in[0,1]^{n_+}}\max_{u:\|u\|_2\leq 1}-\lambda_+^T\diag(b_+)X_+u+\lambda_-^T\diag(b_-)X_-u\\
=&\max_{b_-\in[0,1]^{n_-}}\min_{b_+\in[0,1]^{n_+}}\|-X_+^T\diag(b_+)\lambda_++X_-^T\diag(b_-)\lambda_-\|_2\\
=&\max_{b_-\in[0,1]^{n_-}}\min_{b_+\in[0,1]^{n_+}}\|X_+^T\diag(\lambda_+)b_+-X_-^T\diag(\lambda_-)b_-\|_2.
\end{aligned}
$$
In short, we have
\begin{equation}
\begin{aligned}
    &\max_{u:\|u\|_2\leq 1} |\lambda^T(Xu)_+|\\
    =&\max\big\{\max_{b_+\in[0,1]^{n_+}}\min_{b_-\in[0,1]^{n_-}}\|X_+^T\diag(\lambda_+)b_+-X_-^T\diag(\lambda_-)b_-\|_2,\\
    & \max_{b_-\in[0,1]^{n_-}}\min_{b_+\in[0,1]^{n_+}}\|X_+^T\diag(\lambda_+)b_+-X_-^T\diag(\lambda_-)b_-\|_2\big\}.
\end{aligned}
\end{equation}
This completes the proof.
\end{proof}

\section{Proofs in Section \ref{sec:neg_result}}

\subsection{Proof of Theorem \ref{thm:dual_neg}}
\begin{proof}
We first formulate the approximation of the max-cut problem as a decision problem. 

\texttt{Approx-Max-Cut}
\begin{itemize}
\item Input: data $W\in \mbS^n$, relative error $\epsilon\geq 0$,
\item Goal: find value $q$ such that  
$$
(1+\epsilon)q\geq \max_{z\in \{-1,1\}^n} z^TWz\geq  q.
$$
\end{itemize}
Here $\mbS^n$ is the set of symmetric matrices with size $n\times n$. From \cite{arora1998proof,goemans1995improved}, with $c\leq 1/83$, an algorithm to solve \texttt{Max-Cut} with relative error $\epsilon=c$ would imply that $P=NP$. Suppose that we have a polynomial-time algorithm to solve 
\texttt{Approx-}\texttt{Dual} with relative error $\epsilon$ and the corresponding dual variable is $\tilde \lambda$. Suppose that $\lambda^*$ is the optimal solution to \eqref{maxmargin:dual}. Then, we have
\begin{equation}
    \tilde \lambda^Ty\geq (1-\epsilon)(\lambda^*)^Ty, \quad \max_{z\in[0,1]^n}\| X^T\diag(\tilde{\lambda})z\|_2\leq \sqrt{q_1}\leq 1, 
\end{equation}
Here $q_1$ is an approximation of the optimal value for the max-cut problem 
$$
\begin{aligned}
q_2=&\max_{z\in[0,1]^n}\| X^T\diag(\tilde{\lambda})z\|_2^2\\
    =&\frac{1}{4}\max_{z\in\{-1,1\}^{n+1}}\tilde z^T\bmbm{I\\\bone}\diag(\tilde \lambda)XX^T\diag( \tilde \lambda)\bmbm{I&\bone}\tilde z.
\end{aligned}
$$
Let $\tilde \lambda'=\frac{\tilde \lambda}{\sqrt{q_2}}$. Then, we note that
\begin{equation}
    (\tilde \lambda')^Ty\geq \frac{1-\epsilon}{\sqrt{q_2}} (\lambda^*)^Ty, \max_{z\in[0,1]^n}\| X^T\diag(\tilde{\lambda}')z\|_2=1, 
\end{equation}
From the optimality of $\lambda^*$, we note that $1-\epsilon\leq \sqrt{q_2}$. This implies that 
$$
\frac{q_1}{q_2}\leq (1-\epsilon)^{-2}\leq 1+1/83,
$$
which comes from that $\epsilon\leq 1-\sqrt{83/84}$. Thus, the polynomial algorithm to solve the \texttt{Approx-Dual} problem also solves the Max-Cut problem with relative error $1/83$, which implies $P=NP$. This leads to a contradiction.

For the hinge loss, suppose that we have a polynomial-time algorithm to solve 
\texttt{Approx-Dual-Hinge} with relative error $\epsilon$. Let $\lambda^*$ be the optimal solution to \eqref{maxmargin:dual} and we choose $\beta<\|\lambda^*\|_\infty^{-1}$. In this case, we note that
\begin{equation}
\begin{aligned}
   D^\text{hinge}=\max_{\lambda}&\; \beta(\beta^{-1}\lambda)^Ty, \\
   \text{ s.t. }&0\leq \beta^{-1} \diag(y)\lambda \leq \beta^{-1}, \\
   &\max_{\|w\|_2\leq 1}\beta^{-1}|\lambda^T(Xw)_+|\leq 1.
\end{aligned}
\end{equation}
We note that for $\beta<\|\lambda^*\|_\infty^{-1}$, $\beta\lambda^*$ is feasible for $D^\text{hinge}$ and $\beta(\lambda^*)^Ty=\beta D$. We also note that $D^\text{hinge}\leq \beta D$. This implies that $D^\text{hinge}=\beta D$. Thus, the polynomial-time algorithm to solve \texttt{Approx-Dual-Hinge} with relative error $\epsilon$ can also solve \texttt{Approx-Dual} with relative error $\epsilon$. This leads to a contradiction. 

For the general loss, suppose that we have a polynomial-time algorithm to solve 
\texttt{Approx-Dual-General} with relative error $\epsilon$ and the corresponding dual variable is $\tilde \lambda$. Suppose that $\lambda^*$ is the optimal solution to \eqref{maxmargin:dual_general}. Then, we have
\begin{equation}
    g(\tilde \lambda)\geq (1-\epsilon)g(\lambda^*), \quad \max_{z\in[0,1]^n}\| X^T\diag(\tilde{\lambda})z\|_2\leq \sqrt{q_1}\leq \beta, 
\end{equation}
Here $q_1$ is an approximation of the optimal value for the max-cut problem 
$$
\begin{aligned}
q_2=&\max_{z\in[0,1]^n}\| X^T\diag(\tilde{\lambda})z\|_2^2\\
    =&\frac{1}{4}\max_{z\in\{-1,1\}^{n+1}}\tilde z^T\bmbm{I\\\bone}\diag(\tilde \lambda)XX^T\diag( \tilde \lambda)\bmbm{I&\bone}\tilde z.
\end{aligned}
$$
Let $\tilde \lambda'=\frac{\beta \tilde \lambda}{\sqrt{q_2}}$. Then, we note that
\begin{equation}
    g(\tilde \lambda')\geq \frac{(1-\epsilon)\beta}{\sqrt{q_2}} g(\lambda^*), \max_{z\in[0,1]^n}\| X^T\diag(\tilde{\lambda}')z\|_2=1, 
\end{equation}
From the optimality of $\lambda^*$, we have
$$
\frac{q_1}{q_2}\leq (1-\epsilon)^{-2}\leq 1+1/83,
$$
The rest of the proof is analogous to the one for the max-margin loss.
\end{proof}

\subsection{Proof of Theorem \ref{thm:primal_neg}}

\begin{proof}
Suppose that a polynomial-time algorithm solves \texttt{Approx-Primal} with relative error $\epsilon$ and let $(W_1^*,w_2^*)$ be the approximate ReLU network to the problem \eqref{maxmargin:primal} by the algorithm. Then, we have
\begin{equation}
    \frac{1}{2}(\|W_1^*\|_F^2+\|w_2^*\|_2^2)\leq (1+\epsilon)P_\text{non-cvx}. 
\end{equation}
Let $\mcP=\{\diag(\mbI(Xw_{1,i}\geq 0))|i\in[m]\}$ and we write $\mcP=\{M_1,\dots,M_k\}$. Consider the following optimization problem
\begin{equation}\label{maxmargin:primal_sub}
\begin{aligned}
    \min \; &\sum_{i=1}^k \|u_i\|_2+\|u_i'\|_2\\
    \text{ s.t. } &\diag(y)\sum_{i=1}^k M_iX(u_i-u_i')\geq \bone, (2M_i-I)Xu_i\geq 0,(2M_i-I)Xu_i'\geq 0, i\in[k]. 
\end{aligned}
\end{equation}
Suppose that the optimal value of the above problem is $p_2$. As $(W_1^*,w_2^*)$ renders a feasible solution to \eqref{maxmargin:primal_sub} with an optimal value which is upper bounded by $\frac{1}{2}(\|W_1^*\|_F^2+\|w_2^*\|_2^2)$, we have
\begin{equation}
    p_2\leq \frac{1}{2}(\|W_1^*\|_F^2+\|w_2^*\|_2^2)\leq (1+\epsilon)P_\text{non-cvx}.
\end{equation}
On the other hand, the dual problem of \eqref{maxmargin:primal_sub} is given by
\begin{equation}\label{maxmargin:dual_sub}
    \max \lambda^Ty, \text{ s.t. } \diag(y)\lambda \geq 0, \max_{i\in[k]}\max_{u:\|u\|_2\leq 1, (2M_i-I)Xu\geq 0}|\lambda^TM_iXu|\leq 1,
\end{equation}
Denote the optimal value of the above problem as $d_2$. From the convex duality, we have $p_2=d_2$. Let $\lambda^*$ be the optimal solution of \eqref{maxmargin:dual}. Suppose that
\begin{equation}
    \max_{i\in[k]}\max_{u:\|u\|_2\leq 1, (2D_i-I)Xu\geq 0}|(\lambda^*)^TM_iXu|=c.
\end{equation}
Then, $\lambda^*/c$ is a feasible solution to \eqref{maxmargin:dual_sub}. This implies that \begin{equation*}
    D/c=(\lambda^*/c)^Ty\leq d_2,
\end{equation*}
Hence, we have $c\geq \frac{D}{d_2}=\frac{P_\text{non-cvx}}{p_2}\geq (1+\epsilon)^{-1}$. We note that $$
\max_{i\in[k]}\max_{u:\|u\|_2\leq 1, (2D_i-I)Xu\geq 0}|(\lambda^*)^TD_iXu|^2
$$
can be evaluated in polynomial-time and it gives an approximation of the max-cut problem
\begin{equation}\label{equ:sub_max_cut}
    \max_{u:\|u\|_2\leq 1} |( \lambda^*)^T(Xu)_+|^2
\end{equation} with relative error $(1+\epsilon)^2-1$. With $\epsilon\leq \sqrt{84/83}-1$, this is equivalent to say that the max-cut problem \eqref{equ:sub_max_cut} can be approximated with relative error less than $1/83$ in polynomial time, which contradicts with $P\neq NP$. 

Suppose that a polynomial-time algorithm solves \texttt{Approx-Primal} with relative error $\epsilon$ and let $(H^*, W_1^*,w_2^*)$ be the approximate gated ReLU network to the problem \eqref{maxmargin:primal} by the algorithm. Then, we have
\begin{equation}
    \frac{1}{2}(\|W_1^*\|_F^2+\|w_2^*\|_2^2)\leq (1+\epsilon)P_\text{non-cvx}. 
\end{equation}
Let $\mcP=\{\diag(\mbI(Xh_{i}\geq 0))|i\in[m]\}$ and we write $\mcP=\{M_1,\dots,M_k\}$. Consider the following optimization problem
\begin{equation}\label{maxmargin:primal_sub2}
\begin{aligned}
    \min \; &\sum_{i=1}^k \|u_i\|_2\\
    \text{ s.t. } &\diag(y)\sum_{i=1}^k M_iXu_i\geq \bone.
\end{aligned}
\end{equation}
Suppose that the optimal value of the above problem is $p_3$. Then, we have
\begin{equation}
    p_3\leq  \frac{1}{2}(\|W_1^*\|_F^2+\|w_2^*\|_2^2)\leq (1+\epsilon)P_\text{non-cvx}. 
\end{equation}
On the other hand, the dual problem of \eqref{maxmargin:primal_sub2} is given by
\begin{equation}\label{maxmargin:dual_sub2}
    \max \lambda^Ty, \text{ s.t. } \diag(y)\lambda \geq 0, \max_{i\in[k]}\max_{u:\|u\|_2\leq 1}|\lambda^TM_iXu|\leq 1,
\end{equation}
Denote the optimal value of the above problem as $d_3$. From the convex duality, we have $p_3=d_3$. Let $\lambda^*$ be the optimal solution of \eqref{maxmargin:dual}. Similarly, we note that 
$$
\max_{i\in[k]}\max_{u:\|u\|_2\leq 1}|\lambda^TM_iXu|^2
$$
can be evaluated in poly-nomial time and it gives an approximation of the max-cut problem \eqref{equ:sub_max_cut} with relative error $(1+\epsilon)^2-1$. With $\epsilon\leq \sqrt{84/83}-1$, this is equivalent to say that the max-cut problem \eqref{equ:sub_max_cut} can be approximated with relative error less than $1/83$ in polynomial time, which contradicts with $P\neq NP$.

For the hinge loss or the general loss, similarly, based on the polynomial-time algorithm to solve \texttt{Approx-Primal-Hinge} or \texttt{Approx-Primal-General} with relative error $\epsilon$, we can find a polynomial-time approximation of \eqref{equ:sub_max_cut} with relative error $(1+\epsilon)^2-1$. This leads to a contradiction toward $P\neq NP$.

\end{proof}
\section{Proofs in Section \ref{sec:ortho}}

\subsection{Proof of Proposition \ref{prop:ortho}}

\begin{proof}
We note that 
$$
\begin{aligned}
&\|X_+^T\diag(\lambda_+)b_+-X_-^T\diag(\lambda_-)b_-\|_2^2\\
=&\|X_+^T\diag(\lambda_+)b_+\|_2^2+\|X_-^T\diag(\lambda_-)b_-\|_2^2\\
&-2b_+^T\diag(\lambda_+)X_+X_-^T\diag(\lambda_-)b_-.
\end{aligned}
$$
As $X_+^TX_-\leq 0$, $\lambda_+\geq 0$ and $\lambda_-\geq 0$, we have $b_+^T\diag(\lambda_+)X_+X_-^T\diag(\lambda_-)b_-\leq 0$.
This implies that 
\begin{equation}
    \|X_+^T\diag(\lambda_+)b_++X_-^T\diag(\lambda_-)b_-\|_2^2\geq \|X_+^T\diag(\lambda_+)b_+\|_2^2.
\end{equation}
The equality is achieved when $X_-^T\diag(\lambda_-)b_-=0$. This implies that
\begin{equation}
    \min_{b_-\in[0,1]^{n_-}}\|X_+^T\diag(\lambda_+)b_++X_-^T\diag(\lambda_-)b_-\|_2=\|X_+^T\diag(\lambda_+)b_+\|_2.
\end{equation}
Therefore, we have
\begin{equation}
\begin{aligned}
    &\max_{b_+\in[0,1]^{n_+}}\min_{b_-\in[0,1]^{n_-}}\|X_+^T\diag(\lambda_+)b_+-X_-^T\diag(\lambda_-)b_-\|_2\\
    =&\max_{b_+\in[0,1]^{n_+}}\|X_+^T\diag(\lambda_+)b_+\|_2
\end{aligned}
\end{equation}
We also note that
\begin{equation}
\begin{aligned}
    \|X_+^T\diag(\lambda_+)b_+\|_2^2
    =&b_+^T\diag(\lambda_+)X_+X_+^T\diag(\lambda_+)b_+\\
    \leq &\bone^T\diag(\lambda_+)X_+X_+^T\diag(\lambda_+)\bone.
\end{aligned}
\end{equation}
This is because $\diag(\lambda_+)X_+X_+^T\diag(\lambda_+)\geq 0$. Hence, we have
\begin{equation}
    \max_{b_+\in[0,1]^{n_+}}\|X_+^T\diag(\lambda_+)b_+\|_2 = \|X_+^T\diag(\lambda_+)\bone\|_2=\|X_+^T\lambda_+\|_2.
\end{equation}
This implies that 
    \begin{equation}
        \max_{b_+\in[0,1]^{n_+}}\min_{b_-\in[0,1]^{n_-}}\|X_+^T\diag(\lambda_+)b_+-X_-^T\diag(\lambda_-)b_-\|_2 = \|X_+\lambda_+\|_2.
    \end{equation}
By symmetry, we also have
\begin{equation}
    \max_{b_-\in[0,1]^{n_-}}\min_{b_+\in[0,1]^{n_+}}\|X_+^T\diag(\lambda_+)b_+-X_-^T\diag(\lambda_-)b_-\|_2 = \|X_-\lambda_-\|_2.
\end{equation}
This completes the proof. 
\end{proof}

\subsection{Proof of Lemma \ref{lem:ortho_sol}}
\begin{proof}
Consider the Lagrange function:
$$
L(\lambda_+,\lambda_-,v_+,v_-,u_+,u_-) =  \lambda_+^T\bone+\lambda_-^T\bone-u_+^T(X_+^T\lambda_+-v_+)-u_-^T(X_-^T\lambda_--v_-),
$$
where $\lambda_+\geq 0, \lambda_-\geq 0,\|v_+\|_2\leq 1$ and $\|v_-\|_2\leq 1$. Thus, the dual problem of \eqref{dual:ortho} is given by
\begin{equation}
\begin{aligned}
    &\min_{u_+,u_-}\max_{\lambda_+\geq0,\lambda_-\geq0,\|v_+\|_2\leq 1,\|v_-\|_2\leq 1}L(\lambda_+,\lambda_-,v_+,v_-,u_+,u_-) \\
    =&\min_{u_+,u_-}\pp{\max_{\lambda_+\geq0}\lambda_+^T(\bone-X_+u_+)+\max_{\lambda_-\geq0} \lambda_-^T(\bone-X_-u_-)+\max_{\|v_+\|_2\leq 1}u_+^Tv_++\max_{\|v_+\|_2\leq 1}u_-^Tv_-}\\
    =&\min_{u_+,u_-} \|u_+\|_2+\|u_-\|_2,\text{ s.t. }X_+u_+\geq \bone, X_-u_-\geq \bone. 
\end{aligned}
\end{equation}
Let $u_+,u_-$ be the optimal solution to \eqref{primal:ortho}. From the complementary slackness, there exists $v_+,v_-$ and $\lambda_+,\lambda_-$ such that 
\begin{equation}
    \frac{u_+}{\|u_+\|_2}=v_+, \frac{u_-}{\|u_-\|_2}=v_-, v_+=X_+^T\lambda_+, v_-=X_-^T\lambda_-.
\end{equation}
Therefore, we have 
$$
\begin{aligned}
X_-u_+&=\|u_+\|_2X_-v_+=\|u_+\|_2X_-X_+^T\lambda_+\leq0, \\
X_+u_-&=\|u_-\|_2X_+v_-=\|u_-\|_2X_-X_+^T\lambda_-\leq0, 
\end{aligned}
$$
This completes the proof.

For the hinge loss, consider the Lagrange function:
$$
L(\lambda_+,\lambda_-,v_+,v_-,u_+,u_-,r_+,r_-) =  \lambda_+^T\bone+\lambda_-^T\bone-u_+^T(X_+^T\lambda_+-\beta v_+)-u_-^T(X_-^T\lambda_--\beta v_-)+r_+^T(\bone-\lambda_+)+r_-^T(\bone-\lambda_-). 
$$
where $\lambda_+,\lambda_-,r_+,r_-\geq 0,\|v_+\|_2\leq 1$ and $\|v_-\|_2\leq 1$. Thus, the dual problem of \eqref{dual:ortho_general} for hinge loss is given by
\begin{equation}
\begin{aligned}
    &\min_{r_+,r_-\geq 0,u_+,u_-}\max_{\lambda_+,\lambda_-\geq 0,\|v_+\|_2\leq 1,\|v_-\|_2\leq 1}L(\lambda_+,\lambda_-,v_+,v_-,u_+,u_-) \\
    =&\min_{r_+,r_-\geq 0,u_+,u_-}\pp{\bone^Tr_++\bone^Tr_-+\max_{\lambda_+\geq0}\lambda_+^T(\bone-X_+u_+-r_+)+\max_{\lambda_-\geq0} \lambda_-^T(\bone-X_-u_--r_-)+\beta\max_{\|v_+\|_2\leq 1}u_+^Tv_++\beta\max_{\|v_+\|_2\leq 1}u_-^Tv_-}\\
    =&\min_{r_+,r_-\geq 0,u_+,u_-}\bone^Tr_++\bone^Tr_-+ \beta(\|u_+\|_2+\|u_-\|_2),\text{ s.t. }r_++X_+u_+\geq \bone, r_-+X_-u_-\geq \bone\\
    =&\min_{u_+,u_-\in\mbR^d} \bone^T(\bone-X_+u_+)_++\bone^T(\bone-X_-u_-)_++\beta(\|u_+\|_2+\|u_-\|_2). 
\end{aligned}
\end{equation}
Let $u_+,u_-$ be the optimal solution to \eqref{primal:ortho_general} for hinge loss. Similarly, by complementary slackness, we can show that $X_-u_+\leq 0$ and $X_+u_-\leq 0$.

For the general loss, consider the Lagrange function:
$$
L(\lambda_+,\lambda_-,v_+,v_-,u_+,u_-) =  g(\lambda_+)+g(\lambda_-)-u_+^T(X_+^T\lambda_+-\beta v_+)-u_-^T(X_-^T\lambda_--\beta v_-).
$$
where $\lambda_+,\lambda_-\geq 0,\|v_+\|_2\leq 1$ and $\|v_-\|_2\leq 1$. Thus, the dual problem of \eqref{dual:ortho_general} for hinge loss is given by
\begin{equation}
\begin{aligned}
    &\min_{u_+,u_-}\max_{\lambda_+,\lambda_-\geq 0,\|v_+\|_2\leq 1,\|v_-\|_2\leq 1}L(\lambda_+,\lambda_-,v_+,v_-,u_+,u_-) \\
    =&\min_{u_+,u_-} \ell(X_+u_+)+\ell(X_-u_-)+ \beta (\|u_+\|_2+\|u_-\|_2).
\end{aligned}
\end{equation}
Let $u_+,u_-$ be the optimal solution to \eqref{primal:ortho_general}. Similarly, by complementary slackness, we can show that $X_-u_+\leq 0$ and $X_+u_-\leq 0$.
\end{proof}

\section{Proofs in Section \ref{sec:neg}}
\subsection{Proof of Lemma \ref{lem:sdp_bnd}}
\begin{proof}
We first introduce the Schur Product Theorem. For $X,Y\succeq 0$, $X\circ Y\succeq 0$. Here $\circ$ denotes the element-wise product. We note that
\begin{equation}
    \arcsin(z)=z+\frac{1}{3}z^3+\dots,
\end{equation}
which implies that 
\begin{equation}
    \arcsin(Z)=Z+\frac{1}{3}Z^3+\dots\succeq Z.
\end{equation}
Therefore, we have
\begin{equation}
\begin{aligned}
    &\mbE[z^T\tilde Qz]=\mbE[\tr(\tilde Qzz^T)]=\tr\pp{\tilde Q\mbE[zz^T]}\\
    \stepa{=}&\frac{2}{\pi} \tr(\tilde Q\arcsin(Z_0))\geq \frac{2}{\pi}\tr(\tilde QZ_0)=\frac{2}{\pi}\textbf{SDP}
\end{aligned}
\end{equation}
In step (a) we utilize the identity that $\mbE[zz^T]=\frac{2}{\pi}\arcsin(Z_0)$. 
\end{proof}
\subsection{Proof of Lemma \ref{lem:valid}}
\begin{proof}
For simplicity, we denote $Q=\bmbm{I\\\bone^T}\diag(\tilde \lambda)XX^T\diag( \tilde \lambda)\bmbm{I&\bone}$. The dual problem of \eqref{equ:sdp} is given by
\begin{equation}
    \min_{\zeta\in\mbR^n} \zeta^T\bone, \text{ s.t. } \diag(\zeta)-Q\succeq 0.
\end{equation}
From the complementary slackness condition, there exists $\zeta\in\mbR^n$ such that
\begin{equation}
    \pp{\diag(\zeta)-Q}\tilde Z=0,
\end{equation}
which implies that $\diag(\zeta)\tilde Z = Q\tilde Z$. Without the loss of generality, we can assume that $\tilde \lambda>0$. As $\diag(\zeta)-Q\succeq 0$, for $i\in[n]$, we have
$$
\zeta_i\geq \tilde \lambda_i^2\|x_i\|_2^2>0.
$$
This implies that 
\begin{equation}\label{equ:range}
   \bmbm{I&0}\tilde Z =  \bmbm{\diag(\zeta_{1:n})^{-1}&0} Q\tilde Z
\end{equation}
As $r\sim \mcN(0,\tilde Z)$, there exists $p\in\mbR^{n+1}$ such that $r=\tilde Z p$. We can rewrite $b$ as 
$$
b=\mbI(\text{sign}(r_{n+1})r_{1:n}\geq 0).
$$
As $r_{1:n} = \bmbm{I&0}\tilde Zp$, from \eqref{equ:range}, we note that
\begin{equation}
\begin{aligned}
    r_{1:n} = & \bmbm{\diag(\zeta_{1:n})^{-1}&0} Q\tilde Zp\\
    =&\diag(\zeta_{1:n})^{-1}\diag(\tilde \lambda)XX^T\diag( \tilde \lambda)\bmbm{I&\bone}\tilde Zp
\end{aligned}
\end{equation}
We can let $w=X^T\diag( \text{sign}(r_{n+1})\tilde \lambda)\bmbm{I&\bone}\tilde Zp$. Then, we have
$$
b=\mbI(\text{sign}(r_{n+1})r_{1:n}\geq 0)=\mbI(Xw\geq 0).
$$
This completes the proof. 
\end{proof}

\subsection{Proof of Lemma \ref{lem:prob_bnd}}
Consider the event 
\begin{equation}
    E_{\epsilon_0}=\bbbb{-\epsilon_0 \bone\bone^T\leq \frac{2}{\pi}\frac{1}{k}\sum_{i=1}^kz^i(z^i)^T-\tilde Z\leq \epsilon_0 \bone\bone^T}
\end{equation}
For $j,l\in[n+1]$, we note that $z^i_jz^i_l\in[-1,1]$. We also note that $\var(z^i_jz^i_l)\leq 1$. Therefore, from Hoeffding's inequality, we have
$$
P(E_{\epsilon_0})\geq 1-(n+1)^2e^{-2k\epsilon^2_0}.
$$
For all $\lambda\geq0$, we have
\begin{equation}
     \tr(\tilde Z\bmbm{I\\\bone^T}\diag(\tilde\lambda)XX^T\diag(\tilde \lambda)\bmbm{I&\bone})\geq C_1 (\lambda^T\bone)^2.
\end{equation}
Here we recall that $C_1$ is the optimal value of \eqref{min:sub}. Under event $E_{\epsilon_0}$, we note that
\begin{equation}
\begin{aligned}
    &\tr((Z_k-\tilde Z)\bmbm{\diag(\lambda)X\\\bone^T\diag(\lambda)X}\bmbm{X^T\diag( \lambda)&X^T\diag(\lambda)\bone})\\
    \leq& \epsilon_0\norm{\bmbm{I\\\bone^T}\diag(\lambda)XX^T\diag( \lambda)\bmbm{I&\bone}}_F\\
    \leq& C_2\epsilon_0\norm{\bmbm{I\\\bone^T}\diag(\lambda)\diag( \lambda)\bmbm{I&\bone}}_F\\
    \leq&C_2 \epsilon_0 \pp{\sqrt{\sum_{i=1}^n \lambda_i^4}+2(\lambda^T\bone)\sqrt{\sum_{i=1}^n\lambda_i^2}+(\lambda^T\bone)^2}\\
    \leq &4C_2\epsilon_0(\lambda^T\bone)^2\\
    = &\frac{4C_2\epsilon_0}{C_1} C_1 (\lambda^T\bone)^2\\
    \leq&\frac{4C_2\epsilon_0}{C_1}\tr(\tilde Z\bmbm{I\\\bone^T}\diag(\tilde\lambda)XX^T\diag(\tilde \lambda)\bmbm{I&\bone})
\end{aligned}
\end{equation}
Here we recall that $C_2=\max_{i\in[n]}\|x_i\|_2^2$. Thus, by taking $\epsilon_0=\frac{C_1\epsilon}{4C_2}$ and $k\geq \frac{1}{2\epsilon^2}\log(1/\delta)\log(n+1)^2$, we have $P(E_{\epsilon_0})\geq 1-\delta$ and under event $E_{\epsilon_0}$, the inequality \eqref{inequ:1} holds. This completes the proof. 

\subsection{Proof of Lemma \ref{lem:sdp_minmax}}\label{proof:lem:sdp_minmax}
\begin{proof}
Consider the following problem:
\begin{equation}\label{dual:minmax}
    \min_{\lambda:\lambda^T\bone=1,\lambda\geq 0} \max_{Z:Z\succeq 0, \diag(Z)=\bone}\tr( Z\bmbm{I\\\bone^T}\diag(\lambda)XX^T\diag( \lambda)\bmbm{I&\bone}).
\end{equation}
From Sion's minimax theorem, the above problem is equivalent to
\begin{equation}\label{dual:maxmin}
    \max_{Z:Z\succeq 0, \diag(Z)=\bone}\min_{\lambda:\lambda^T\bone=1,\lambda\geq 0} \tr( Z\bmbm{I\\\bone^T}\diag(\lambda)XX^T\diag( \lambda)\bmbm{I&\bone}).
\end{equation}
We note that $\lambda_0=\tilde \lambda/(\tilde \lambda^T\bone)$ is the optimal solution to \eqref{dual:minmax}. This implies that
$$
\begin{aligned}
&\min_{\lambda:\lambda^T\bone=1,\lambda\geq 0} \max_{Z:Z\succeq 0, \diag(Z)=\bone}\tr( Z\bmbm{I\\\bone^T}\diag(\lambda)XX^T\diag( \lambda)\bmbm{I&\bone})\\
=&(\tilde \lambda^T\bone)^{-2} \tr( \tilde Z\bmbm{I\\\bone^T}\diag(\tilde \lambda)XX^T\diag( \tilde \lambda)\bmbm{I&\bone})=4(\tilde \lambda^T\bone)^{-2}.
\end{aligned}
$$
Therefore, $\tilde Z$ is the optimal solution \eqref{dual:maxmin}. It also implies that
\begin{equation}\label{min:sub}
\begin{aligned}
    \min_{\lambda:\lambda^T\bone=1,\lambda\geq 0} \tr(\tilde Z\bmbm{I\\\bone^T}\diag(\lambda)XX^T\diag( \lambda)\bmbm{I&\bone})
    =4(\tilde \lambda^T\bone)^{-2}.
\end{aligned}
\end{equation}
Let $\lambda_1$ be the optimal solution to \eqref{dual:sdp_sub}. Then, we note that $\lambda_1/\lambda_1^T\bone$ is the optimal solution to \eqref{min:sub}. This implies that
\begin{equation}
    4(\tilde \lambda^T\bone)^{-2} = 4(\lambda_1^T\bone)^{-2}.
\end{equation}
Hence, we have $\lambda_1\bone=\tilde \lambda^T\bone$. This implies that the optimal value of \eqref{dual:sdp_sub} is equal to \eqref{dual:sdp}. This completes the proof. 
\end{proof}

We then prove that \eqref{dual:sdp_hinge} is equivalent to \eqref{dual:sdp_sub_hinge} for $\beta<\max_{u:\|u\|_2\leq 1} |y^T(Xu)_+|$. We can write the optimal value of \eqref{dual:sdp_hinge} as a function of $\beta$ and denote it by $c(\beta)$. We first show that $c(\beta)$ is strictly increasing when $\beta<\max_{u:\|u\|_2\leq 1} |y^T(Xu)_+|$. Otherwise, there exist $\beta_1<\beta_2<\max_{u:\|u\|_2\leq 1} |y^T(Xu)_+|$ such that $c(\beta_1)=c(\beta_2)$. As $c(\beta_1)=c(\beta_2)$, this means that the constraint that
$$
\max_{Z:Z\succeq 0, \diag(Z)=\bone}\tr( Z\bmbm{I\\\bone^T}\diag(\lambda)XX^T\diag( \lambda)\bmbm{I&\bone})\leq \beta_2^2
$$
is inactive. This implies that $c(\beta_1)=c(\beta_2)=n$. We also have
$$
\beta_1^2>\max_{Z:Z\succeq 0, \diag(Z)=\bone}\tr( Z\bmbm{I\\\bone^T}\diag(y)XX^T\diag( y)\bmbm{I&\bone})
$$
However, we note that
$$
\beta_1^2<\max_{u:\|u\|_2\leq 1} |y^T(Xu)_+|^2\leq \max_{Z:Z\succeq 0, \diag(Z)=\bone}\tr( Z\bmbm{I\\\bone^T}\diag(y)XX^T\diag(y)\bmbm{I&\bone})<\beta_1^2.
$$
This leads to a contradiction. 

This implies that the function $c(\beta)$ is strict decreasing for $\beta<\max_{u:\|u\|_2\leq 1} |y^T(Xu)_+|$. Consider the following problem:
\begin{equation}\label{dual:minmax_hinge}
    \min_{\lambda:\lambda^T\bone=c(\beta),0\leq \lambda\leq \bone} \max_{Z:Z\succeq 0, \diag(Z)=\bone}\tr( Z\bmbm{I\\\bone^T}\diag(\lambda)XX^T\diag( \lambda)\bmbm{I&\bone}).
\end{equation}
From Sion's minimax theorem, the above problem is equivalent to
\begin{equation}\label{dual:maxmin_hinge}
    \max_{Z:Z\succeq 0, \diag(Z)=\bone}\min_{\lambda:\lambda^T\bone=c,0\leq \lambda\leq \bone} \tr( Z\bmbm{I\\\bone^T}\diag(\lambda)XX^T\diag( \lambda)\bmbm{I&\bone}).
\end{equation}
We note that $\tilde \lambda$ is the optimal solution to \eqref{dual:minmax_hinge}. This also implies that $\tilde Z$ is the optimal solution to \eqref{dual:maxmin_hinge}. Let $\lambda_1$ be the optimal solution to \eqref{dual:sdp_sub} and let the corresponding optimal value be $c_1$. Then, we have $c_1\leq c(\beta)$. We note that $\lambda_1$ is the optimal solution to
\begin{equation}
    \min_{\lambda:\lambda^T\bone=c_1,0\leq \lambda\leq \bone} \tr( \tilde Z\bmbm{I\\\bone^T}\diag(\lambda)XX^T\diag( \lambda)\bmbm{I&\bone}).
\end{equation}
Moreover, its optimal value is $\beta$. On the other hand, we note that 
\begin{equation}
\begin{aligned}
    \beta=&\min_{\lambda:\lambda^T\bone=c_1,0\leq \lambda\leq \bone} \tr( \tilde Z\bmbm{I\\\bone^T}\diag(\lambda)XX^T\diag( \lambda)\bmbm{I&\bone})\\
\leq&\min_{\lambda:\lambda^T\bone=c_1,0\leq \lambda\leq \bone} \max_{Z:Z\succeq 0, \diag(Z)=\bone}\tr( Z\bmbm{I\\\bone^T}\diag(\lambda)XX^T\diag( \lambda)\bmbm{I&\bone})=:\beta_1.
\end{aligned}
\end{equation}
We note that $c(\beta_1)=c_1$. As $\beta_1\geq \beta$, this implies that $c_1\geq c(\beta)$. Thus, we have $c_1=c(\beta)$. This completes the proof. 

\subsection{Proof of Proposition \ref{prop:oracle}}
\begin{proof}
For any $\lambda_0\in\mbR^n$, we compute $c_2(\lambda_0)$ by solving the SDP. If $c_2(\lambda_0)\leq 1$, this implies that $\lambda\in S$. Otherwise, we have $c_2(\lambda_0)\geq 1$.  We can construct a hyperplane separation $\lambda_0$ and $S$ as follows. Consider the hyperplane $H=\{\lambda|c_2'(\lambda_0)(\lambda-\lambda_0)=0\}$. We can compute $c_2'(\lambda_0)$ based on the maximizer $Z$ in calculating $c_2(\lambda_0)$. For fixed $Z\succeq0$, 
$$
\tr\pp{Z\bmbm{I\\\bone^T}\diag(\lambda)XX^T\diag(\lambda)\bmbm{I&\bone}}
$$
is a convex function of $\lambda$. Hence, $c_2(\lambda)$ is a convex function of $\lambda$. If $c_2'(\lambda_0)(\lambda-\lambda_0)>0$, then from the convexity of $c_2(\lambda)$, we have 
$$
c_2(\lambda)\geq c_2(\lambda_0)+c_2'(\lambda_0)(\lambda-\lambda_0)>c_2(\lambda_0)> 1.
$$
This implies that $ c_2(\lambda)> 1$. Therefore, the hyperplane $H$ separates $S$ and $\lambda_0$. 
\end{proof}

\section{Proofs in Section \ref{sec:geo}}

\subsection{Proof of Proposition \ref{prop:general}}
\begin{proof}
Without the loss of geneality, we can assume that $c^*>1$. From \eqref{equ:tri}, we note that $\lambda^*$ satisfies that
\begin{equation}
\begin{aligned}
   1\geq&\max_{b_+ \in [0,1]^{n_+} }  \min_{ b_- \in  [0,1]^{n_-} }  \| X_+^T\diag(\lambda^*_+)b_+ + X_-^T(\lambda^*_-)b_- \|_2\\
   \geq& \max_{b_+\in[0,1]^{n_+}}\|X_+^T\diag(\lambda_+^*)b_+\|_2-\max_{b_-\in[0,1]^{n_-}}\|X_-^T\diag(\lambda_-^*)b_-\|_2\\
   =&\frac{c^*-1}{c^*}\max_{b_+\in[0,1]^{n_+}}\|X_+^T\diag(\lambda_+^*)b_+\|_2.
\end{aligned}
\end{equation}
This implies that $\max_{b_+\in[0,1]^{n_+}}\|X_+^T\diag(\lambda_+^*)b_+\|_2\leq \frac{c^*}{c^*-1}=1-c^*\leq 1-c$. We also have
$$
\max_{b_-\in[0,1]^{n_-}}\|X_-^T\diag(\lambda_-^*)b_-\|_2 = (c^*)^{-1}\max_{b_+\in[0,1]^{n_+}}\|X_+^T\diag(\lambda_+^*)b_+\|_2\leq \frac{1}{c^*-1}\leq \frac{1}{c^{-1}-1}.
$$
Therefore, $(1-c)\lambda^*$ is feasible for $D_c^1$. This implies that
\begin{equation}
     D_c^1\geq (1-c) D.
\end{equation}
On the other hand, let $\tilde \lambda$ be the optimal solution to $D_c^1$. We note that 
\begin{equation}
\begin{aligned}
    &\max_{b_+\in[0,1]^{n_+}}\min_{b_-\in[0,1]^{n_-}}\|X_+^T\diag(\tilde \lambda_+)b_++X_-^T\diag(\tilde \lambda_-)b_-\|_2
    \leq & \max_{b_+\in[0,1]^{n_+}}\|X_+^T\diag(\tilde \lambda_+)b_+\|_2\leq 1
\end{aligned}
\end{equation}
\begin{equation}
\begin{aligned}
    &\max_{b_-\in[0,1]^{n_-}}\min_{b_+\in[0,1]^{n_+}}\|X_+^T\diag(\tilde \lambda_+)b_++X_-^T\diag(\tilde \lambda_-)b_-\|_2
\leq & \max_{b_-\in[0,1]^{n_-}}\|X_-^T\diag(\tilde \lambda_-)b_-\|_2\leq 1
\end{aligned}
\end{equation}
This implies that $\tilde \lambda$ is feasible to $D$. Therefore, we have $D_c^1\leq D$.
\end{proof}

\subsection{Proof of Theorem \ref{thm:geo_bnd} for hinge loss and general loss}\label{app:geo_bnd}
For the hinge loss, consider
\begin{equation}\label{equ:dc1_hinge}
\begin{aligned}
     D_c^1=&\max \lambda^Ty 
    \text{ s.t. } &0\leq \diag(y)\lambda \leq \bone,
    &\max_{b_+\in\{0,1\}^{n_+}} \|X_+^T\diag(\lambda_+)u\|_2\leq \beta ,  
    &\max_{b_-\in\{0,1\}^{n_-}} \|X_-\diag(\lambda_-)v\|_2\leq \beta c.
\end{aligned}
\end{equation}
\begin{equation}
\begin{aligned}
     D_c^2=&\max \lambda^Ty 
    \text{ s.t. } &0\leq \diag(y)\lambda \leq \bone,
    &\max_{b_+\in\{0,1\}^{n_+}} \|X_+^T\diag(\lambda_+)u\|_2\leq c\beta,  
    &\max_{b_-\in\{0,1\}^{n_-}} \|X_-\diag(\lambda_-)v\|_2\leq \beta.
\end{aligned}
\end{equation}
Suppose that $0<c\leq \min\{c^*,(c^*)^{-1}\}$. Then, if $c^*>1$, we have the approximation ratio
\begin{equation}
\begin{aligned}
(1-c)D^\text{hinge} \leq  D_c^1\leq D^\text{hinge}.
\end{aligned}
\end{equation}
If $c^*<1$, we have the approximation ratio
\begin{equation}
\begin{aligned}
(1-c)D^\text{hinge} \leq  D_c^2\leq D^\text{hinge}
\end{aligned}
\end{equation}
The rest of the proof is analogous to the one for the max-margin problem. 

For the general loss, consider
\begin{equation}\label{equ:dc1_general}
\begin{aligned}
     D_c^1=&\max g(\lambda)
    \text{ s.t. } &\diag(y)\lambda \geq 0,
    &\max_{b_+\in\{0,1\}^{n_+}} \|X_+^T\diag(\lambda_+)u\|_2\leq \beta ,  
    &\max_{b_-\in\{0,1\}^{n_-}} \|X_-\diag(\lambda_-)v\|_2\leq \beta c.
\end{aligned}
\end{equation}
\begin{equation}
\begin{aligned}
     D_c^2=&\max g(\lambda)
    \text{ s.t. } & \diag(y)\lambda \geq 0,
    &\max_{b_+\in\{0,1\}^{n_+}} \|X_+^T\diag(\lambda_+)u\|_2\leq c\beta,  
    &\max_{b_-\in\{0,1\}^{n_-}} \|X_-\diag(\lambda_-)v\|_2\leq \beta.
\end{aligned}
\end{equation}
Suppose that $0<c\leq \min\{c^*,(c^*)^{-1}\}$. Then, if $c^*>1$, we have the approximation ratio
\begin{equation}
\begin{aligned}
(1-c)D^\text{gen} \leq  D_c^1\leq D^\text{gen}.
\end{aligned}
\end{equation}
If $c^*<1$, we have the approximation ratio
\begin{equation}
\begin{aligned}
(1-c)D^\text{gen} \leq  D_c^2\leq D^\text{gen}
\end{aligned}
\end{equation}
The rest of the proof is analagous to the one for the max-margin problem. 

\bibliographystyle{plain}
\bibliography{Newton.bib}
\end{document}